\documentclass{TEAI}
\usepackage{mathtools}
\usepackage{wrapfig}
\usepackage{needspace}
\usepackage{enumitem}
\usepackage{tabularx}
\usepackage{makecell}
\usepackage{xurl}
\usepackage[export]{adjustbox}
\newtheorem{proposition}{Proposition}[section]
\newtheorem{definition}[proposition]{Definition}

\newcommand{\metricname}{BTE}

\newcommand{\metric}{\mathrm{\metricname}}
\DeclareMathOperator{\JSD}{JSD}
\newcommand{\KL}{D_{\mathrm{KL}}}
\newcommand{\Var}{\mathrm{Var}}

\newcommand{\simplex}{\Delta^{|V|-1}}
\newlength{\attributionvisualheight}
\newsavebox{\attributionimagebox}
\newsavebox{\attributiontablebox}
\newlength{\attributionrowextra}
\newcommand{\gain}[1]{\,{\scriptsize\textcolor{gray}{+#1}}}
\hypersetup{
  pdftitle={Belief-Trajectory Energy: Measuring the Path to a Prediction},
  pdfauthor={Jiahao Ying, Wei Tang, Boxian Ai, Yaoning Wang, Haotian Chen, Wenhe Sun, Caijun Xu, Haozhan Cai, Changyi Xiao, Yixin Cao}
}
\title{Belief-Trajectory Energy: Measuring the Path to a Prediction}
\author[1]{Jiahao Ying}
\author[2]{Wei Tang}
\author[1]{Boxian Ai}
\author[1]{Yaoning Wang}
\author[1]{Haotian Chen}
\author[1]{Wenhe Sun}
\author[1]{Caijun Xu}
\author[1]{Haozhan Cai}
\author[1]{Changyi Xiao}
\author[1,3, \dagger]{Yixin Cao}
\affiliation[1]{Fudan University}
\affiliation[2]{University of Science and Technology of China}
\affiliation[3]{Shanghai Innovation Institute}
\correspondence{\email{jhying@fudan.edu.cn, yxcao@fudan.edu.cn}}
\checkdata[Pages]{\url{https://yingjiahao14.github.io/BTE-web/}}
\abstract{Large language models (LLMs) progressively revise their predictions across Transformer layers, yet we typically observe only the final output, discarding the trajectory through which it is formed. We introduce \emph{Belief-Trajectory Energy} (BTE), a model-grounded measure that characterizes an input through the layerwise predictive revisions it induces in a model. By mapping intermediate states into a shared predictive space, BTE provides a principled measure of belief change that can be summarized as either a scalar or a structured depth profile. Theoretically, we show that local BTE corresponds to predictive revision under the Fisher–Rao geometry, while the sequence of revisions captures information beyond the initial-to-final belief change. Empirically, scalar BTE provides a model-relative signal of difficulty across diverse reasoning tasks, while richer BTE representations support human–LLM review detection and fine-grained generator attribution, reaching up to $0.998$ macro-AUROC and $95.6\%$ eight-way attribution accuracy. Further analysis shows that BTE develops throughout pretraining and is selectively reshaped by targeted training, demonstrating that the resulting measurement reflects what the scoring model has learned. Together, our results establish belief trajectories as a principled model-grounded signal and suggest a broader perspective in which learned models can themselves serve as instruments for characterizing the data they process.
}
\begin{document}
\maketitle
\begingroup
\renewcommand{\thefootnote}{\ensuremath{\dagger}}
\footnotetext[0]{Corresponding author.}
\endgroup
\section{Introduction}
\label{sec:intro}

Large language models (LLMs) internalize a model of the linguistic world through autoregressive training at scale, yet we interact with this model almost exclusively through its final outputs, e.g., a probability, a generation, or a judgment~\citep{Zhao_2026, cao2025toward}. Behind each output, however, lies a sequence of internal computations: as information propagates through Transformer layers, the model progressively transforms and refines its internal state before arriving at the final prediction. Conventional evaluation observes only this endpoint, leaving the trajectory of internal change largely unused.
We posit that this discarded trajectory can provide a useful signal beyond the final output. On the one hand, it records how a given text interacts with the linguistic and world knowledge encoded in the model’s parameters, potentially revealing distinctions that are obscured at the model's output alone. On the other hand, it provides a model-grounded view of the input: rather than relying solely on predefined human-centric properties, which may not align with the model’s own representation, we characterize an input by the predictive revisions it induces in the model.

Turning this rich internal information into a usable signal, however, is not trivial. \textbf{First}, the measurement should be principled rather than engineered: a single forward pass produces a large collection of high-dimensional intermediate states, admitting countless possible representations and reductions. Without a principled formulation, any resulting statistic risks being an engineered feature rather than a meaningful measure of the underlying trajectory of the model.
\textbf{Second}, the signal should be useful beyond a particular downstream predictor. High-dimensional internal representations can support powerful supervised classifiers, but their meaning is often tied to task- and domain-specific training. A broadly useful measurement should instead carry interpretable structure on its own while remaining readily applicable across diverse tasks and domains. \textbf{Third}, the information should be read at the right granularity. Neuron-level analyses are fine-grained but costly and often model-specific, whereas output-level statistics collapse the computation to its endpoint. The challenge is therefore to find an intermediate representation that preserves the evolution of internal computation while remaining compact, comparable, and easy to compute.

To address these challenges, we make the internal trajectory measurable through the model’s own predictive space. Specifically, we read each intermediate state through the model’s vocabulary projection as a next-token predictive distribution, which we refer to as the model’s \emph{belief}. This provides a common, prediction-grounded space in which changes across layers have a consistent interpretation. We then introduce \emph{Belief-Trajectory Energy} (BTE) to measure these successive belief revisions. Aggregating the layerwise revisions across tokens yields a compact depth profile of how the model’s belief evolves throughout computation.
Our theoretical analysis shows that BTE is not an arbitrary measure of internal change. Locally, Jensen–Shannon divergence induces the Fisher–Rao geometry of the predictive distribution, giving each transition a principled interpretation as an effective predictive revision. Across depth, BTE retains intermediate belief revisions that cannot be recovered from the initial-to-final change alone.
Empirically, we examine whether the model-grounded information captured by BTE can serve as a general measurement signal across different settings. \textbf{1)} As a training-free scalar, BTE consistently reflects response difficulty across diverse mathematical and reasoning tasks. Importantly, the same inputs can induce different BTE orderings under different scoring models, suggesting that BTE captures a model-relative aspect of difficulty that is complementary to human-defined difficulty labels. \textbf{2)} Retaining the full BTE depth profile provides richer structural information, supporting lightweight supervised readouts for human–LLM review detection and generator attribution. On paper-disjoint review detection across eight conference-years, Qwen-based BTE achieves a macro-averaged AUROC of $0.998$, while fine-grained BTE representations achieve $95.6\%$ accuracy for eight-way source attribution. \textbf{3)} Beyond these applications, our training analysis shows how such signals develop and change with the scoring model. BTE structure progressively emerges during pretraining and is selectively reshaped by targeted fine-tuning, demonstrating that the resulting measurement reflects what the scorer has learned and the distributions to which it has been exposed.
Together, these results highlight the central perspective behind BTE: a learned model can serve not only as an object of evaluation, but also as a measurement instrument, providing a complementary view of data through its internal predictive dynamics.

\section{The \metricname{} Statistic}
\label{sec:method}

We first map the model’s internal trajectory into predictive space, where states across depth share a common interpretation. Following the logit-lens formulation~\citep{nostalgebraistInterpretingGPTLogit2020a}, we read out each intermediate state through the model’s own output head as a predictive distribution.
Formally, let $\mathcal{M}$ be an open-weight Transformer (the \emph{scoring model}) with $L$ layers, vocabulary $V$, unembedding matrix $W_U\in\mathbb{R}^{|V|\times d}$, and final normalization $\mathrm{LN}$. Given a prompt $x$ and a response $y=(y_1,\dots,y_T)$, a single forward pass over $[x;y]$ yields the residual state $h_t^{(\ell)}$ at each response position $t$ and model depth $\ell$. We read out each intermediate state as:

\begin{equation}
  p_t^{(\ell)} \;=\;
  \mathrm{softmax}\ \bigl(W_U\,\mathrm{LN}(h_t^{(\ell)})\bigr)\in\simplex ,
  \label{eq:lens}
\end{equation}
where $p_t^{(\ell)}$ represents the predictive belief read out from layer $\ell$ at position $t$. The sequence
$p_t^{(0)} \to p_t^{(1)} \to \cdots \to p_t^{(L)}$ forms the model’s \emph{belief trajectory} for token $t$.  Given this trajectory, we quantify the magnitude of its layer-to-layer revisions to characterize how strongly the model’s belief evolves when processing a given input. We refer to this quantity as \emph{Belief Trajectory Energy} (BTE).

\begin{definition}[Belief Trajectory Energy]
\label{def:metric}
For each token $t$ and layer transition $\ell\rightarrow\ell+1$, we have
\begin{equation}
  \metric_{t,\ell}
  \;=\;
  \JSD \bigl(p_t^{(\ell)},\,p_t^{(\ell+1)}\bigr).
  \label{eq:def}
\end{equation}
\end{definition}

The quantity $\metric_{t,\ell}$ is the elementary BTE associated with one
 step of the belief. It measures the magnitude of the predictive-belief revision induced by layer $\ell$. Here, we establish
 further properties.

\begin{figure}[tbp]
    \centering
    \includegraphics[
        width=\textwidth,
        keepaspectratio
    ]{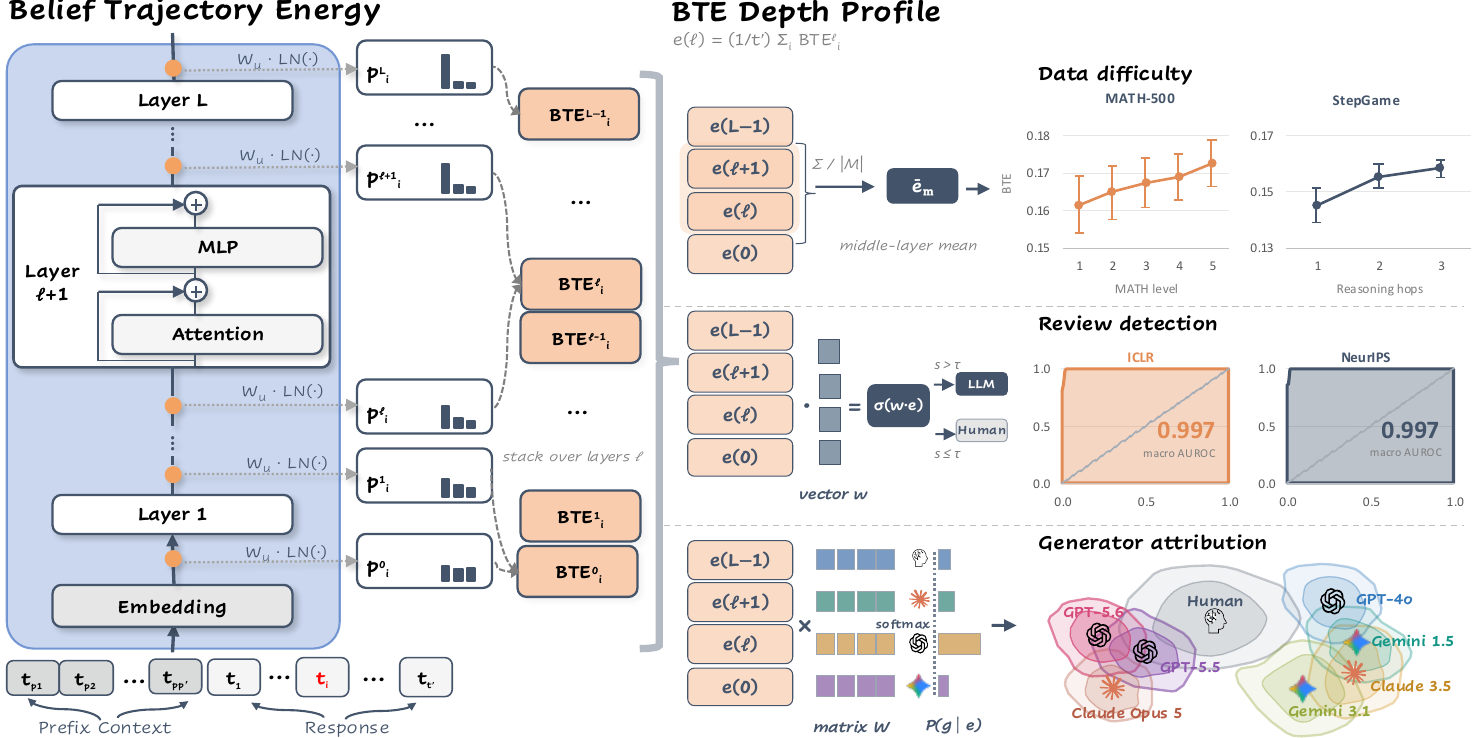}
\caption{\textbf{Overview of Belief-Trajectory Energy (BTE).}
BTE quantifies the layerwise predictive revisions induced by an input and aggregates them into a depth profile. The profile can be further summarized into a scalar for training-free data difficulty measurement, and used directly to distinguish human-LLM reviews and attribute responses to generators.}
     \label{fig:method}

\end{figure}

\subsection{Interpreting Belief Trajectory Energy}
\label{sec:score-physics}

Since BTE is built from successive belief revisions across model depth, we analyze it from two complementary perspectives. \emph{Locally}, we characterize what a single BTE step measures and motivate the use of Jensen–Shannon divergence. \emph{Across depth}, we characterize what the sequence of local revisions captures beyond the initial-to-final belief change alone.
Consider two consecutive predictive beliefs $p=\mathrm{softmax}(z)$ and $q=\mathrm{softmax}(z+\delta)$, where $\delta$ denotes the change in logits. Locally, Jensen–Shannon divergence admits that:

\begin{proposition}[Belief-weighted update variance]
\label{prop:variance}
Let $F=\mathrm{diag}(p)-pp^{\top}$. Then:
\begin{equation}
  \JSD(p,q)
  \;=\;
  \tfrac{1}{8}\,\delta^{\top}F\,\delta+O(\|\delta\|^{3})
  \;=\;
  \tfrac{1}{8}\,\Var_{i\sim p}\bigl[\delta_i\bigr]+O(\|\delta\|^{3}).
  \label{eq:variance}
\end{equation}
\end{proposition}

Proposition~\ref{prop:variance} shows that BTE measures the update relative to the model’s current predictive belief: a uniform shift of all logits contributes nothing, whereas differential updates among tokens carrying substantial probability mass contribute more. Thus, a local BTE step measures the \textbf{effective predictive revision} induced by a layer. This characterization also explains why \eqref{eq:variance} provides a principled measure of such revisions. The quadratic form
$F=\mathrm{diag}(p)-pp^\top$ is the Fisher–Rao metric in logit coordinates, the canonical statistically invariant geometry of probability distributions \citep{cencov2000statistical,lebanon2004extended}. Thus, JSD locally measures belief revision in the intrinsic geometry of the predictive distribution rather than by raw logit distance. Moreover, $\JSD(p,q)$ has a direct information-theoretic interpretation as the mutual information between a fair before/after indicator and a token sampled from the corresponding distribution, giving each BTE step a natural unit in nats.

The local analysis characterizes what each individual BTE step measures; we next ask what is gained by retaining these revisions across model depth.
Unlike output-level measures such as likelihood or perplexity, which depend only on the final predictive distribution, a belief trajectory records the layerwise revisions through which that change is formed rather than collapsing the trajectory into the endpoint displacement $\JSD (p_t^{(0)},p_t^{(L)})$. Formally, since $\sqrt{\JSD}$ is a metric, the triangle inequality together with Cauchy–Schwarz gives
\begin{equation}
  \sum_{\ell=0}^{L-1}\JSD\bigl(p_t^{(\ell)},p_t^{(\ell+1)}\bigr)
  \;\ge\;\tfrac{1}{L}\,\JSD\bigl(p_t^{(0)},\,p_t^{(L)}\bigr).
  \label{eq:energy}
\end{equation}

The endpoint displacement therefore provides only a lower bound on the cumulative layerwise revisions. Even with the same initial and final beliefs, different intermediate trajectories can yield different BTE sequences. BTE thus preserves how the model’s predictive belief evolves across depth, rather than only the final prediction or its net initial-to-final change (In Section~\ref{sec:experiments} we validate it experimentally).

\subsection{BTE Depth Profile}
\label{sec:bte-representation}

The preceding analysis motivates preserving the depth-wise structure of BTE rather than collapsing it into a single endpoint statistic. At the response level, however, token-level BTE values form a $T\times L$ matrix, where the token dimension $T$ varies with response length. To facilitate its subsequent use, we therefore aggregate over tokens while retaining the depth dimension, yielding a fixed-length representation that we term the \emph{BTE depth profile} $\mathbf{e}$ (illustration of the
 calculation of the depth profile is shown in Figure~\ref{fig:method} ):

\begin{equation}
  e(\ell) \;=\; \frac{1}{T}\sum_{t=1}^{T}\metric_{t,\ell},
    \qquad
\mathbf{e}
=
\bigl(e(0),e(1),\ldots,e(L-1)\bigr)
\in \mathbb{R}^{L}.
\label{eq:depth_profile}
\end{equation}

Each entry $e(\ell)$ captures the average predictive revision per token at a particular model depth, removing variation in response length while retaining the depth-wise structure of BTE.
Mean pooling provides a natural compression along the token dimension. Local BTE values are measured on a common and bounded scale, $0\leq\metric_{t,\ell}\leq\log 2$, making them directly aggregable across token positions. Taking their mean further gives $e(\ell)$ a simple interpretation as the expected predictive revision of a token at depth $\ell$, while normalizing out response length. Empirically, this aggregation is also supported by the distributed nature of BTE: the signal is not dominated by a small number of high-energy tokens, with the top $5\%$ accounting for only ${\sim}8\%$ of the observed difference between responses. Mean pooling therefore summarizes the broadly distributed token-level signal. Beyond mean pooling, Section~\ref{sec:attribution} further explores richer token-level statistics to retain information discarded by averaging.

\section{Experimental Analysis}
\label{sec:experiments}

Here, we empirically investigate whether BTE provides a useful measurement signal, proceeding from its scalar magnitude to the structure of its depth profile. First, motivated by the interpretation of BTE as effective predictive revision, we examine whether its magnitude reflects the difficulty an input poses to the scoring model in a fully training-free setting, using BTE directly as a scalar score (Section~\ref{sec:difficulty}). Second, we examine whether the depth profile captures systematic differences in how the model processes inputs of different origins, using human–LLM review detection as a testbed (Section~\ref{sec:review-detection}). Finally, we ask whether these patterns are sufficiently distinctive to identify the specific model that generated an input, evaluating BTE on generator attribution (Section~\ref{sec:attribution}). Unless otherwise specified, we use Llama-3.1-8B-Instruct~\citep{grattafiori2024llama3herdmodels} as the scoring model, whose homogeneous layer structure provides a clean setting for analyzing belief revisions.

\Needspace{22\baselineskip}
\subsection{BTE as a Difficulty Indicator}
\label{sec:difficulty}

Intuitively, if BTE captures the effective predictive revisions induced during computation, inputs that are more challenging to the model may elicit stronger revisions. We therefore here study how BTE varies with difficulty across model depth, without fitting any task-specific predictor.

\begin{wrapfigure}{r}{0.47\textwidth}
  \centering
  \captionsetup{skip=4pt}
  \includegraphics[width=\linewidth]{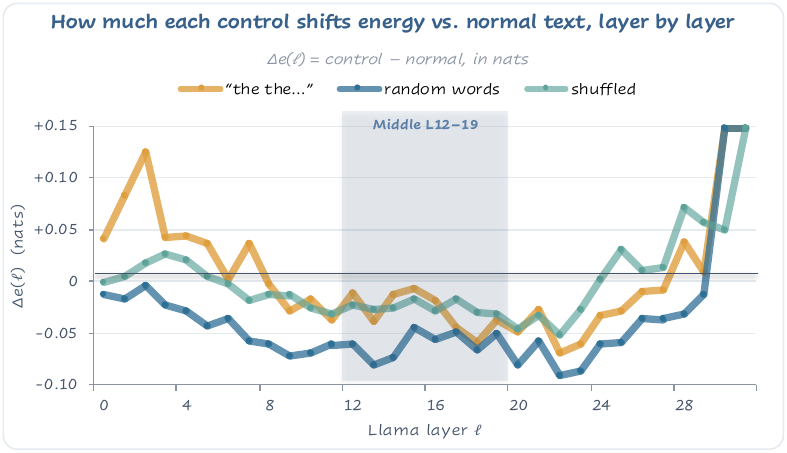}
  \caption{\textbf{Response-disruption control.}
  Layer-wise BTE change relative to normal responses.}
  \label{fig:case-study}
\end{wrapfigure}
Using BTE as a scalar requires reducing its depth profile to a fixed readout. Rather than averaging indiscriminately across all layers, we use the central quarter of transitions (layers 12–19 for our 32-layer Llama scorer) as a fixed, training-free summary. This choice is motivated by an exploratory response-disruption analysis: relative to normal responses, disrupted responses tend to exhibit lower BTE in this central window but substantially higher BTE in the final transitions (Figure~\ref{fig:case-study}). Because these perturbations leave the underlying problem unchanged while increasing terminal-layer BTE, high BTE near the output is not specific to problem difficulty, and averaging over the full depth may conflate qualitatively different effects. We therefore adopt the central window as a conservative readout of intermediate predictive revision and test whether it tracks model-relative difficulty. We further compare central-, late-, and full-depth readouts in Appendix~\ref{app:difficulty}.

\textbf{Experimental Setup.} Using this score, we evaluate difficulty sensitivity in two settings. \textbf{Within-task}, inputs come from a single benchmark with predefined difficulty levels, and we test whether BTE increases with difficulty. We consider MATH-500 (levels 1–5)~\citep{hendrycksmath2021}, ProofWriter (proof depths 1–5)~\citep{tafjord2021proofwriter}, StepGame (reasoning hops 1–5)~\citep{shi2022stepgame}, and SpaRTUN (hops 1–3)~\citep{mirzaee2022transfer}, using official reference solutions as responses. \textbf{Cross-task}, we compare benchmarks of increasing difficulty within the same domain: GSM8K $\rightarrow$ MATH-500 $\rightarrow$ OlympiadBench~\citep{he2024olympiadbench} for mathematical reasoning and StepGame $\rightarrow$ SpaRTUN for spatial reasoning. By default, we score the reference solution conditioned on the problem, as solving difficulty may be expressed in the reasoning required to produce a solution rather than in the problem statement alone.
We use the mean BTE over the central layer set $\mathcal{L}_M$ as the scalar score $\bar e_M$.
In both settings, we report the level-wise change in $\bar e_M$ relative to the easiest level, with 95\% cluster-bootstrap confidence intervals, together with the item-level Spearman correlation $\rho$ between $\bar e_M$ and difficulty.
Because BTE provides a model-grounded rather than input-intrinsic measure, we further examine the signal along two dimensions. \textbf{First}, we compare response BTE with problem-only BTE to determine whether difficulty is already reflected in processing the problem or becomes more apparent when processing its solution. \textbf{Second}, we vary the scoring model to examine how the difficulty signal depends on the model processing the input.

\begin{figure}[tbp]
    \centering
    \includegraphics[
        width=\textwidth,
        keepaspectratio
    ]{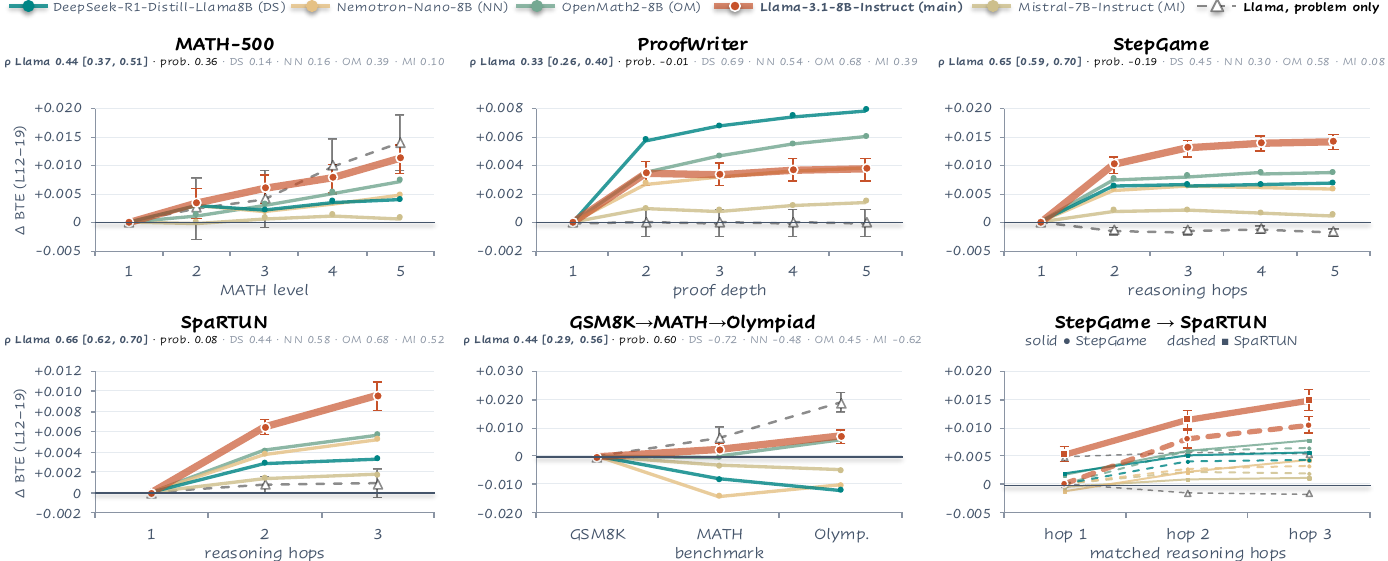}
\caption{Level-wise change in $\bar e_{M}$ relative to the easiest level, within task (top row, SpaRTUN) and across tasks (GSM8K$\to$MATH-500$\to$OlympiadBench, StepGame$\to$SpaRTUN). Solid: reference solution scored given its problem, five scorers ordered red$\to$teal by strength; dashed: problem statement only (Llama). Bars: 95\% bootstrap CI (Llama).}
     \label{fig:diff}
\end{figure}

Shown in Figure~\ref{fig:diff}, we find that
\textbf{1)} \textbf{Response BTE tracks difficulty within and across tasks.}
Using reference solutions as the scoring target and Llama as the scorer, $\bar e_M$ increases consistently with difficulty on all four within-task benchmarks, with item-level Spearman correlations ranging from $\rho=0.33$ to $0.66$ and all confidence intervals excluding zero. BTE also recovers the expected cross-task ordering for both mathematical and spatial reasoning. The cross-task result is less stable across scorers, however: three of the five scorers assign lower BTE to MATH-500 than to GSM8K.
We hypothesize that this reversal may partly reflect scorer-specific training exposure that shifts the overall BTE scale of a distribution. Intuitively, greater exposure to a distribution during training may make its inputs more familiar to the scorer, reducing the predictive revisions they induce even when the underlying problems remain difficult.
This provides a possible explanation for why models with enhanced competition-level mathematical reasoning capabilities, such as DeepSeek-R1-Distill-Llama-8B and Nemotron-Nano-8B, can invert the expected ordering between GSM8K and MATH-500. We test this hypothesis directly in Section~\ref{sec:ablation}. More broadly, this observation suggests a potential use of BTE as a probe of training exposure or model familiarity: if greater exposure systematically suppresses predictive revisions, unusually low BTE may indicate greater familiarity with a given data distribution. We leave a systematic investigation of this possibility to future work.

\textbf{2)}
\textbf{In contrast, a problem statement need not directly reveal the computation required to solve it.} This is particularly evident on ProofWriter and SpaRTUN, where difficulty is controlled by proof depth or reasoning hops while the problem statements remain largely comparable across difficulty levels. Accordingly, problem-only BTE shows little association with difficulty ($\rho=-0.01$ and $0.08$), whereas response BTE exhibits substantially stronger trends ($\rho=0.33$ and $0.66$, respectively). In contrast, on MATH-500, difficulty is already reflected in the problem statement: problem-only BTE achieves $\rho=0.36$, compared with $\rho=0.44$ when scoring the reference solution. These results highlight that BTE does not define a single notion of difficulty, but provides complementary model-grounded views depending on what is scored. Problem-only BTE captures difficulty expressed directly in the input but may miss computation that becomes apparent only during solving; response BTE exposes such computation but is inevitably affected by properties of the particular solution being scored. Here we explore more and detailed in Appendix~\ref{app:difficulty}.

\textbf{3)} Difficulty sensitivity depends on the scoring model. In general, both the strongest model DeepSeek-R1-Distill-Llama and the weakest model Mistral show weaker difficulty responses than intermediate-capability scorers, with mean correlations of $0.43$ and $0.27$, respectively, compared with $0.52$ for Llama.
Intuitively, a model with insufficient capability may fail to differentiate inputs of varying difficulty, whereas a model that is highly capable or particularly familiar with a domain may exhibit less differentiated predictive revisions across the same inputs, compressing their BTE differences. This interpretation is consistent with the cross-task results above, where scorer-specific training exposure can alter the relative BTE scale across benchmarks (see Section~\ref{sec:ablation} for more information).
This pattern also suggests that \textbf{BTE captures a model-relative aspect of difficulty that is complementary to human-defined difficulty labels.} The same input can induce different levels of predictive revision depending on the model processing it, making the scoring model part of the measurement itself. Practically, when using BTE to assess difficulty, the scorer should therefore be selected with respect to both its capability level and its potential training exposure to the evaluated data.

\subsection{BTE as a Review Detector}
\label{sec:review-detection}

The preceding experiments show that even a scalar summary of BTE can provide a useful training-free signal for difficulty. However, collapsing BTE across depth necessarily discards where predictive revisions occur. We next ask whether the full BTE depth profile contains additional structure that can support more fine-grained distinctions about inputs.
We use human-LLM review detection to test this question. Reviews typically combine a synthesis of the paper with evidence-based assessments of its strengths and weaknesses, providing substantial text in which depth-wise patterns can emerge.
Human and LLM reviews of the same paper share an underlying subject while having known provenance labels, reducing variation due to topic.

\textbf{Experimental Setup.}
We use real peer reviews collected and extended from ICLR and NeurIPS papers from 2021–2024 provided by~\citet{yu2025your} (detailed in Appendix~\ref{app:review-data}).
For each review, we evaluate two depth profile representations: \emph{shape-only 32D}, which normalizes each review's profile across depth, and \emph{full 32D}, which retains its absolute values and shape.
We evaluate three training settings: \textbf{A (legacy-to-new transfer)} trains on original human reviews and reviews from the three earlier generators on ICLR 2021 papers; the four newer generators are unseen during training. \textbf{B (cross-family transfer)} uses the same ICLR 2021 papers but holds out one generator family at a time: a classifier trained on the other two families is evaluated on the held-out family, rotating among GPT, Gemini, and Claude. \textbf{C (full coverage)} trains on papers from all eight conference-years and all seven generators, while keeping every test paper disjoint from training. We report both Original human reviews and a Rewrite condition in which only the human reviews are rewritten at test time; the generated reviews and fitted classifiers remain fixed.
We compare against existing detection baselines following work~\citep{yu2025your}, with implementation details provided in Appendix~\ref{app:review-data}.

\begin{table}[!htbp]
\centering
\setlength{\tabcolsep}{1.4pt}
\renewcommand{\arraystretch}{1.12}
\fontsize{8}{9.5}\selectfont
\begin{adjustbox}{max width=\linewidth}
\begin{tabular}{@{}l*{16}{c}@{}}
\toprule
\multirow{2}{*}{\textbf{Method}}
& \multicolumn{8}{c}{\textbf{Original human reviews}}
& \multicolumn{8}{c}{\textbf{GPT-5.5 rewrite of human reviews}} \\
\cmidrule(lr){2-9}\cmidrule(l){10-17}
& \makecell{GPT\\4o} & \makecell{Gem\\1.5}
& \makecell{Clau\\3.5} & \makecell{GPT\\5.5}
& \makecell{GPT\\5.6} & \makecell{Gem\\3.1}
& \makecell{Opus\\5} & \textbf{Avg.}
& \makecell{GPT\\4o} & \makecell{Gem\\1.5}
& \makecell{Clau\\3.5} & \makecell{GPT\\5.5}
& \makecell{GPT\\5.6} & \makecell{Gem\\3.1}
& \makecell{Opus\\5} & \textbf{Avg.} \\
\midrule

Log-likelihood (Llama-3.1-8B)
& 0.945 & 0.994 & 0.989 & 0.351 & 0.119 & 0.965 & 0.565 & 0.704
& 0.926 & 0.992 & 0.986 & 0.251 & 0.066 & 0.953 & 0.453 & 0.661 \\
Log-likelihood (Qwen-3.5-9B)
& 0.950 & 0.996 & 0.997 & 0.799 & 0.436 & 0.992 & 0.903 & 0.868
& 0.931 & 0.995 & 0.996 & 0.731 & 0.330 & 0.990 & 0.866 & 0.834 \\
SpecDetect (raw)~\citep{luo2026specdetect}
& 0.439 & 0.441 & 0.297 & 0.526 & 0.588 & 0.139 & 0.810 & 0.463
& 0.455 & 0.456 & 0.313 & 0.539 & 0.601 & 0.150 & 0.817 & 0.476 \\
SpecDetect (normalized)
& 0.013 & 0.003 & 0.008 & 0.708 & 0.934 & 0.098 & 0.639 & 0.343
& 0.019 & 0.003 & 0.012 & 0.805 & 0.970 & 0.142 & 0.758 & 0.387 \\
Binoculars~\citep{hans2024spotting}
& 0.908 & 0.995 & 0.993 & 0.593 & 0.160 & 0.957 & 0.711 & 0.760
& 0.874 & 0.991 & 0.989 & 0.521 & 0.112 & 0.942 & 0.644 & 0.725 \\
Open-Anchor~\citep{yu2025your}
& 0.949 & 0.987 & 0.972 & 0.891 & 0.779 & 0.924 & 0.803 & 0.901
& 0.946 & 0.985 & 0.970 & 0.884 & 0.766 & 0.920 & 0.792 & 0.895 \\

\midrule
\multicolumn{17}{@{}l}{\textit{Setting A: three legacy training generators; seven held-out conference-years}} \\
\quad Llama-3.1-8B (shape-only, 32D)
& 0.990 & 0.982 & 0.991 & 0.614 & 0.378 & 0.939 & 0.183 & 0.725
& 0.989 & 0.981 & 0.991 & 0.587 & 0.348 & 0.932 & 0.161 & 0.713 \\
\rowcolor{gray!10}
\quad Llama-3.1-8B (full, 32D)
& 0.992 & 0.985 & 0.982 & 0.678 & 0.543 & 0.904 & 0.109 & 0.742
& 0.991 & 0.982 & 0.978 & 0.631 & 0.490 & 0.883 & 0.084 & 0.720 \\
\quad Qwen-3.5-9B (shape-only, 32D)
& 0.995 & 0.976 & 0.984 & 0.606 & 0.315 & 0.847 & 0.154 & 0.697
& 0.995 & 0.974 & 0.983 & 0.586 & 0.297 & 0.835 & 0.141 & 0.687 \\
\rowcolor{gray!10}
\quad Qwen-3.5-9B (full, 32D)
& \textbf{0.996} & 0.990 & 0.989 & 0.858 & 0.725 & 0.961 & 0.717 & 0.891
& \textbf{0.996} & 0.989 & 0.988 & 0.820 & 0.657 & 0.952 & 0.646 & 0.864 \\

\midrule
\multicolumn{17}{@{}l}{\textit{Setting B: leave-one-family-out training; seven held-out conference-years}} \\
\quad Llama-3.1-8B (shape-only, 32D)
& 0.946 & 0.811 & 0.975 & 0.893 & 0.817 & 0.916 & 0.902 & 0.894
& 0.943 & 0.835 & 0.977 & 0.889 & 0.810 & 0.927 & 0.907 & 0.898 \\
\rowcolor{gray!10}
\quad Llama-3.1-8B (full, 32D)
& 0.935 & 0.924 & 0.984 & 0.897 & 0.857 & 0.944 & 0.824 & 0.909
& 0.932 & 0.924 & 0.982 & 0.891 & 0.850 & 0.944 & 0.810 & 0.905 \\
\quad Qwen-3.5-9B (shape-only, 32D)
& 0.990 & 0.995 & 0.997 & 0.998 & 0.996 & 0.997 & 0.993 & 0.995
& 0.987 & 0.995 & 0.997 & 0.998 & 0.995 & 0.996 & 0.992 & 0.994 \\
\rowcolor{gray!10}
\quad Qwen-3.5-9B (full, 32D)
& 0.995 & 0.996 & 0.998 & 0.991 & 0.985 & 0.997 & 0.996 & 0.994
& 0.995 & 0.995 & 0.998 & 0.990 & 0.983 & 0.997 & 0.995 & 0.993 \\

\midrule
\multicolumn{17}{@{}l}{\textit{Setting C: seven training generators; all eight conference-years}} \\
\quad Llama-3.1-8B (shape-only, 32D)
& 0.915 & 0.971 & 0.989 & 0.957 & 0.930 & 0.973 & 0.951 & 0.955
& 0.913 & 0.970 & 0.989 & 0.957 & 0.929 & 0.973 & 0.951 & 0.954 \\
\rowcolor{gray!10}
\quad Llama-3.1-8B (full, 32D)
& 0.943 & 0.979 & 0.992 & 0.979 & 0.964 & 0.969 & 0.951 & 0.968
& 0.938 & 0.977 & 0.991 & 0.977 & 0.961 & 0.967 & 0.947 & 0.965 \\
\quad Qwen-3.5-9B (shape-only, 32D)
& 0.993 & 0.997 & \textbf{0.999} & 0.998 & 0.997 & 0.998 & 0.997 & 0.997
& 0.993 & 0.997 & \textbf{0.999} & 0.998 & 0.997 & 0.998 & 0.997 & 0.997 \\
\rowcolor{gray!10}
\quad Qwen-3.5-9B (full, 32D)
& 0.995 & \textbf{0.998} & \textbf{0.999} & \textbf{0.999}
& \textbf{0.998} & \textbf{0.999} & \textbf{0.999} & \textbf{0.998}
& 0.995 & \textbf{0.998} & \textbf{0.999} & \textbf{0.999}
& \textbf{0.998} & \textbf{0.999} & \textbf{0.999} & \textbf{0.998} \\
\bottomrule
\end{tabular}
\end{adjustbox}

\normalsize
\caption{Human-LLM review detection under Original and strict-rewrite human-review conditions. BTE classifiers are trained under three paper-disjoint settings (A-C); all methods are evaluated on the same seven target conference-years. AUROC is macro-averaged over conference-year $\times$ generator cells; higher is better. Values are rounded to three decimal places; \textbf{boldface} marks the best value in each column based on unrounded scores. Column abbreviations: Gem-1.5 = Gemini-1.5, Clau-3.5 = Claude-3.5, Gem-3.1 = Gemini-3.1. More detail is shown in Appendix~\ref{app:review-data}.}
\label{tab:bte-review-three-settings}
\end{table}

Table~\ref{tab:bte-review-three-settings} reveals that:
\textbf{1) BTE profiles provide a transferable and style-robust representation.}
Across Settings A and B, classifiers trained on BTE profiles generalize across papers and conference-years (year-wise results in Appendix~\ref{app:year-ablation} show that this transfer is broadly consistent, although performance is relatively weaker on 2024 reviews, where the distinction between purely human- and LLM-assisted writing may itself be less clear), and remain effective even when the target generator family is entirely held out during training. The signal is also largely preserved under the Rewrite condition, suggesting that the separation is not primarily driven by superficial human–LLM stylistic cues. In particular, under cross-family transfer (Setting B), Qwen-BTE achieves a macro-AUROC of $0.995$, with essentially unchanged performance after rewriting ($0.994$).
This transfer, however, does not extend uniformly across generations of LLMs. When trained only on legacy generators (Setting A), performance drops substantially on GPT-5.5, GPT-5.6, and Opus5, particularly for the Llama scorer. The depth profiles help explain this shift (Figure~\ref{fig:gen-profiles}).
Under both scorers, earlier generators exhibit lower BTE than human reviews across multiple depth regions. This gap progressively narrows for newer generators, particularly in the middle and final layers, and is eventually reversed by Opus-5. Notably, Qwen retains a clear separation in the early–middle layers (AUROC $0.82$–$0.95$), explaining its stronger cross-generation transfer.
\textbf{Importantly, this trajectory shift is accompanied by a substantial increase in human-evaluated review quality.} In pairwise comparisons against matched human reviews (detailed setting in Appendix~\ref{app:human-eval}), GPT-5.5, GPT-5.6, and Opus-5 achieve a pooled preference score of $0.878$, well above human parity ($0.5$) and the $0.531$ of the three earlier generators. Across all seven generators, human preference is strongly correlated with the signed middle-layer BTE shift relative to human reviews ($\rho=0.857$, $p=0.024$). Thus, generators receiving higher human preference tend to exhibit BTE profiles closer to, and for Opus-5 beyond, the human distribution.
These results provide a possible explanation for why detection becomes harder for newer generators: \textbf{improvements in human-evaluated review quality are accompanied by BTE profiles that increasingly approach the human distribution. More broadly, this shift suggests that the question may no longer be only how to distinguish model-generated reviews from human ones.} As strong models produce reviews that are consistently preferred by human evaluators, their reviewing behavior may itself provide a useful reference for potentially improving reviewing.

\begin{figure}[tbp]

    \centering
    \includegraphics[
        width=\textwidth,
        keepaspectratio
    ]{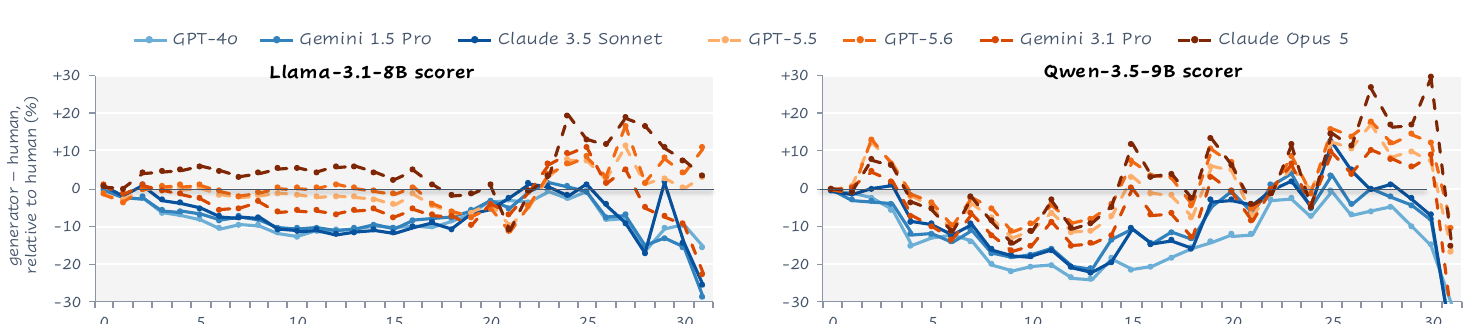}
\caption{Depth profiles of generated reviews relative to human reviews of the same papers, per layer, under the Llama-3.1-8B (left) and Qwen-3.5-9B (right) scorers. }
     \label{fig:gen-profiles}
     
\end{figure}

\textbf{2) BTE provides strong discriminative results.}
With all seven generators represented in training (Setting C), Qwen-BTE reaches $0.998$ macro-AUROC and Llama-BTE reaches $0.968$, compared with $0.901$ for Open-Anchor and below $0.868$ for likelihood-based baselines. Performance remains nearly unchanged after rewriting human reviews. Open-Anchor relies on semantic similarity to reference generations and can therefore degrade when the anchor models differ substantially from the generators being detected, while likelihood-based methods reduce the model’s response to output-level statistics. In contrast, BTE retains both the magnitude and depth-wise structure of predictive revisions, providing richer discriminative information beyond the final predictive output.
The scoring model also matters. Qwen-3.5-9B consistently outperforms Llama-3.1-8B, particularly under the cross-family and full-coverage settings. A familiarity measurement points to the reason: under Qwen, the mean next-token negative log-likelihood of generated reviews is 10.4\% lower than that of human reviews of the same papers, against 2.4\% under Llama, and the mid-depth BTE deficit of generated reviews is correspondingly deeper (detailed in Appendix~\ref{app:scorer-ablation}). Qwen is thus substantially more familiar with model-generated text than with human text, which widens the very gap the detector exploits. A plausible cause is heavier exposure to model-generated text during training, although the training mixtures are not public. \textbf{Interestingly, this provides a complementary view of our difficulty results (Section~\ref{sec:difficulty}): scorer familiarity can reduce comparability across data distributions, but when familiarity itself differs systematically between classes, it can provide a finer-grained signal for discrimination.}

\subsection{BTE for Generator Attribution}
\label{sec:attribution}

The preceding experiments show that the BTE depth profile contains sufficient provenance information to distinguish human- from LLM-generated reviews. We next ask whether BTE can support a finer-grained distinction: identifying which model generated a review. Beyond the depth profile, we additionally retain token-level BTE patterns that are removed by token averaging, allowing us to test whether variation across both tokens and depth provides additional information for generator attribution (detailed setting in Appendix~\ref{app:attribution-transfer}).
As shown in Table~\ref{tab:attribution-features}, the BTE depth profile alone already supports strong eight-way attribution, reaching $86.4\%$ and $91.3\%$ accuracy with Llama-3.1-8B-Instruct and Qwen-3.5-9B, respectively. Retaining token-level variation further improves performance to $93.2\%$ and $95.6\%$. The remaining errors are concentrated between GPT-5.5 and GPT-5.6, which are confused with each other in roughly 10\% of cases (Figure~\ref{fig:generator-clusters}; more results are shown in Figure~\ref{fig:generator_clusters_umap_all}), while the other sources are largely well separated. These results show that generator-specific information exists at multiple granularities of the belief trajectory, with a practical trade-off between preserving finer-grained information and maintaining a compact, efficient representation.

\begin{table}[!htbp]
\vspace{5mm}
\centering
\setlength{\parskip}{0pt}
\setlength{\attributionrowextra}{0pt}

\newcommand{\attributionrowstrut}{%
  \rule[-\dimexpr\dp\strutbox+\attributionrowextra/2\relax]{0pt}%
       {\dimexpr\ht\strutbox+\dp\strutbox+\attributionrowextra\relax}%
}

\newcommand{\attributiontabular}{%
  \small
  \renewcommand{\arraystretch}{1}%
  \setlength{\tabcolsep}{2pt}%
  \begin{tabular*}{0.49\textwidth}{@{\extracolsep{\fill}}lrll@{}}
    \toprule
    \attributionrowstrut
    \textbf{Per-layer features}
      & \textbf{Dim.} & \textbf{Llama} & \textbf{Qwen} \\
    \midrule
    \attributionrowstrut Token mean
      & 32 & 86.4 & 91.3 \\
    \attributionrowstrut\quad + std, P10/ P50/ P90
      & 160 & 91.8\gain{5.4} & 94.4\gain{3.1} \\
    \attributionrowstrut\quad + 4 position-bin means
      & 288 & 92.7\gain{0.9} & 95.0\gain{0.6} \\
    \attributionrowstrut\quad + AC, top-10\% mean
      & 352 & \textbf{93.2}\gain{0.5}
            & \textbf{95.6}\gain{0.6} \\
    \bottomrule
  \end{tabular*}%
}

\sbox{\attributionimagebox}{%
  \includegraphics[width=0.49\textwidth]
    {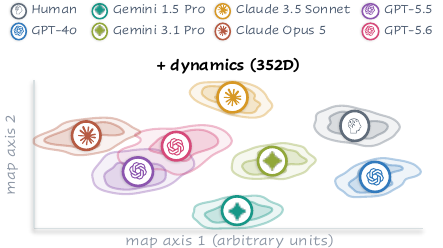}%
}

\setlength{\attributionvisualheight}{%
  \dimexpr\ht\attributionimagebox+\dp\attributionimagebox\relax
}

\sbox{\attributiontablebox}{\attributiontabular}

\setlength{\attributionrowextra}{%
  \dimexpr(
    \attributionvisualheight
    -\ht\attributiontablebox
    -\dp\attributiontablebox
  )/5\relax
}

\sbox{\attributiontablebox}{\attributiontabular}

\begin{minipage}[t]{0.49\textwidth}
  \vspace{0pt}
  \usebox{\attributiontablebox}
\end{minipage}\hfill
\begin{minipage}[t]{0.49\textwidth}
  \vspace{0pt}
  \usebox{\attributionimagebox}
\end{minipage}

\par\nointerlineskip\vspace{6pt}
\noindent
\begin{minipage}[b]{0.49\textwidth}
  \captionsetup{skip=0pt}
  \captionof{table}{Eight-way source attribution
  (Human + seven generators) as token-level statistics
  are cumulatively added to the per-layer BTE
  representation; gray numbers are gains over the
  previous row. Accuracy on generated reviews from
  held-out papers under the full-coverage setting;
  details and further settings in
  Appendix~\ref{app:attribution-transfer}.}
  \label{tab:attribution-features}
\end{minipage}\hfill
\begin{minipage}[b]{0.49\textwidth}
  \captionsetup{skip=0pt}
  \captionof{figure}{Held-out reviews of the eight
  sources in the classifier's discriminant space
  (352D features, Llama scorer), shown with a
  nonlinear (t-SNE) map; contours enclose 90\%
  and 50\% of each source's density, dots are
  centroids.}
  \label{fig:generator-clusters}
\end{minipage}
\end{table}

\section{Ablation Study}
\label{sec:ablation}

Our results suggest that BTE provides a model-grounded measurement whose behavior depends on the scoring model. We therefore investigate how this signal develops and changes as the model learns. We study two complementary stages of training: broad pretraining, which progressively builds the model’s representation of the data distribution, and targeted fine-tuning, which selectively increases exposure to particular regions of that distribution. Together, these analyses allow us to examine both how BTE emerges and how it is reshaped by what the model has learned.
 For the pretraining analysis, we trace OLMo-2-7B~\citep{OLMo2} across training checkpoints. For the targeted-exposure analysis, we focus on mathematical reasoning, which provides controlled difficulty labels and related data distributions with different levels of model exposure. Specifically, we fine-tune Llama-3.1-8B-Instruct on human-verified solutions from PRM800K~\citep{lightman2023lets}. Training details and additional robustness analyses are provided in Appendix~\ref{app:sft-details}.

\textbf{1) BTE structure emerges progressively during pretraining.}
We trace BTE across the pretraining process. As shown in Figure~\ref{fig:sft-shape}, the randomly initialized model exhibits virtually no input-dependent variation: the five MATH difficulty levels collapse onto the same BTE profile, with identical middle-layer means of $0.013$ and a standard deviation below $0.0002$ across the 500 solutions. As pretraining progresses, the profiles become increasingly differentiated while BTE rises overall, particularly in the middle and late layers; middle-layer BTE increases seven-fold from $0.013$ to $0.093$ in OLMo-2-7B~\citep{OLMo2}. In parallel, BTE becomes increasingly aligned with input difficulty: its correlation with MATH difficulty shifts from $\rho=-0.13$ at 1T tokens to $\rho=+0.03$ at the end of OLMo pretraining at 4T tokens, while Llama-3.1, pretrained on 15T tokens, reaches $\rho=+0.36$.
The concurrent development of BTE magnitude and difficulty sensitivity suggests that \textbf{the structure captured by BTE is not simply induced by the architecture itself, but develops as the model learns the data distribution.} As pretraining exposes the model to increasingly plenty of data, its predictive trajectories begin to reflect learned distinctions among inputs, which in turn become observable through BTE.
\begin{figure}[tbp]
\centering
\includegraphics[width=\textwidth]{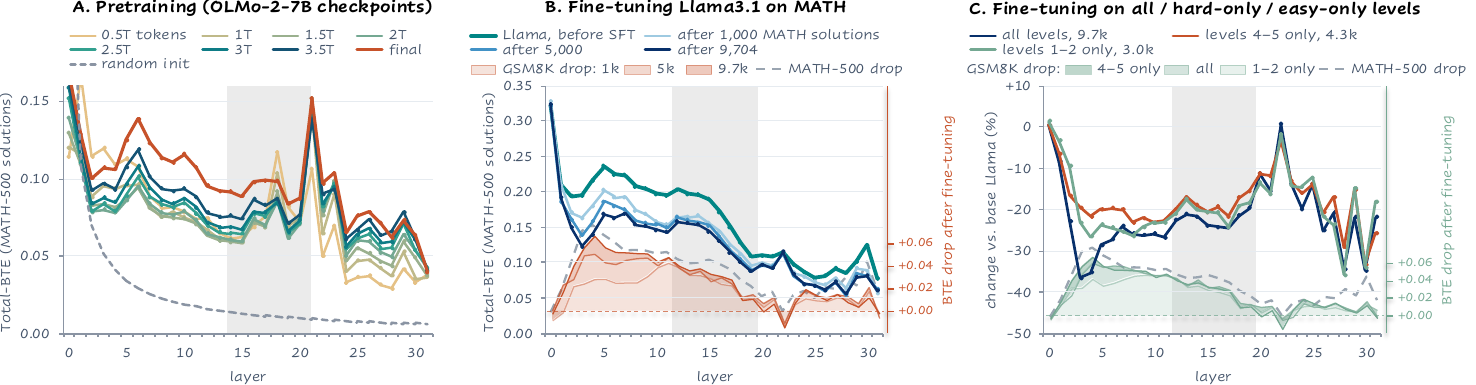}

\caption{Familiarity and the depth profile. \textbf{A}: OLMo-2-7B stage-1 checkpoints from 0.5T to the 4T final model. \textbf{B}: mean profile before and after fine-tuning on 1k, 5k and 9.7k MATH solutions; shading gives the GSM8K drop at each dose and the dashed line the MATH-500 drop (right axis). \textbf{C}: relative change per layer for the three training arms, with the GSM8K drop per arm shaded and the MATH-500 drop dashed as in A.  Grey band: layers 12-19.}
\label{fig:sft-shape}

\end{figure}

\textbf{2)}
Fine-tuning on a specific distribution systematically reduces the BTE it induces. After MATH fine-tuning, middle-layer BTE decreases by $22\%$ on MATH-500, compared with $17\%$ on GSM8K, eventually reversing their cross-benchmark ordering (Figure~\ref{fig:sft-shape}B). The same pattern appears within MATH: fine-tuning selectively on levels4–5 reduces BTE more strongly for level5 than for level~1 ($-19.7\%$ vs.\ $-17.5\%$). These results, aligned with the hypothesis in Section~\ref{sec:difficulty}, suggest that greater targeted exposure is associated with smaller predictive revisions on the exposed distribution.
Together with the pretraining results, this analysis shows that BTE reflects what the scoring model has learned: its structure develops during broad pretraining and shifts with subsequent targeted exposure. \textbf{BTE therefore provides a model-relative view of an input that complements properties defined independently of the scorer.} Consequently, the choice and training history of the scoring model form part of the measurement itself.

\section{Related Work}

A large body of work studies the internal computation of language models through hidden states, attention patterns, neurons, and sparse features~\citep{dai-etal-2022-knowledge, ying2025beyond, gurneefinding, ghandeharioun2024patchscopes, cao2025model, ying2025disentangling}. Logit-lens methods provide a particularly relevant perspective by projecting intermediate residual states through the model’s output head~\citep{nostalgebraistInterpretingGPTLogit2020a,belrose2023eliciting, chuang2024dola}. Subsequent work has used intermediate states to study knowledge recall, reasoning, and model computation. Our work shifts the objective from interpreting individual intermediate states to measuring their successive changes.
 More recent work has explored internal representations and activations as richer signals for estimating uncertainty~\citep{cao2025model,kadavath2022language, azaria2023internal}, difficulty, correctness, or other properties of model inputs and outputs~\citep{wang2025effieval,ying2025beyond}. These approaches demonstrate the value of looking beyond final predictions, but often rely on high-dimensional representations or task-specific supervised readouts. BTE instead operates directly in the model’s predictive distribution space and provides both a training-free scalar summary and a compact depth-wise representation of how predictive beliefs evolve during a forward pass.

\section{Discussion}

The central perspective behind BTE is to measure data from the model’s learned view of the world. Many properties of data are defined independently of the model, whereas BTE provides a complementary, model-grounded perspective by characterizing an input through the predictive revisions it induces. Consequently, BTE should not be expected to reproduce human-defined properties exactly: as our difficulty experiments illustrate, the resulting measurement can vary with both the scoring model and what is being scored. \textbf{From this perspective, scorer dependence is not merely a source of variation, but part of the measurement itself: different models can provide different views of the same data.}
More broadly, model-grounded measurement through internal model dynamics remains an emerging direction, with many open questions about how internal computation should be represented and measured. Nevertheless, the consistent patterns observed across our experiments suggest that this perspective is promising and merits further investigation.

\section{Conclusion}

We introduced Belief-Trajectory Energy, a model-grounded measure that characterizes an input through the predictive revisions it induces across Transformer layers. BTE provides both a simple training-free signal and a structured representation, showing consistent utility for difficulty measurement, human–LLM review detection, and generator attribution. Our theoretical and empirical analyses further show that BTE reflects how an input is modeled from the scorer’s perspective and evolves with the model’s training and familiarity. More broadly, our results suggest a shift in perspective: learned models can serve not only as objects of evaluation, but also as measurement instruments whose internal dynamics provide a complementary view of the data they process. BTE represents one step toward this direction.

\bibliographystyle{plainnat}
\bibliography{main}
\clearpage
\appendix
\section{Supplementary experimental result}
\label{app:theory}

\subsection{Additional Theoretical Analysis}
\label{app:proof}

\begin{proof}
Let $q=p+\varepsilon$, where $\sum_i \varepsilon_i=0$, and define
$M=(p+q)/2=p+\varepsilon/2$. Expanding around $\varepsilon=0$ gives
\begin{align}
\KL(p\Vert M)
&=
-\sum_i p_i
\log\left(1+\frac{\varepsilon_i}{2p_i}\right) \\
&=
\frac{1}{8}\sum_i\frac{\varepsilon_i^2}{p_i}
+O(\|\varepsilon\|^3),
\end{align}
where the first-order term vanishes because
$\sum_i\varepsilon_i=0$. Similarly,
\begin{equation}
\KL(q\Vert M)
=
\frac{1}{8}\sum_i\frac{\varepsilon_i^2}{p_i}
+O(\|\varepsilon\|^3).
\end{equation}
Hence,
\begin{equation}
\JSD(p,q)
=
\frac{1}{8}
\varepsilon^\top
\mathrm{diag}(1/p)
\varepsilon
+O(\|\varepsilon\|^3).
\end{equation}

For $p=\mathrm{softmax}(z)$ and
$q=\mathrm{softmax}(z+\delta)$, the first-order expansion of the
softmax gives
\begin{equation}
\varepsilon
=
F\delta+O(\|\delta\|^2),
\qquad
F=\mathrm{diag}(p)-pp^\top.
\end{equation}
Using
\begin{equation}
F\,\mathrm{diag}(1/p)\,F=F,
\end{equation}
we obtain
\begin{equation}
\JSD(p,q)
=
\frac{1}{8}\delta^\top F\delta
+O(\|\delta\|^3).
\end{equation}
Finally,
\begin{align}
\delta^\top F\delta
&=
\sum_i p_i\delta_i^2
-
\left(\sum_i p_i\delta_i\right)^2 \\
&=
\Var_{i\sim p}[\delta_i],
\end{align}
which proves the result.
\end{proof}

\subsection{Difficulty Experiments: Data and Readout Windows}
\label{app:difficulty}

\paragraph{Data}
Table~\ref{tab:difficulty-data} lists the benchmarks used in Section~\ref{sec:difficulty}. For every item the scorer reads the problem followed by the official solution, and BTE is computed on the solution tokens only. ProofWriter items are grouped by theory (100 theories, five proof depths each) and SpaRTUN items by story, which defines the clusters used for bootstrap resampling; the other have one item per cluster.

\begin{table}[htbp]
\centering
\small
\adjustbox{max width=\linewidth}{%
\begin{tabular}{llrrr}
\toprule
\textbf{Benchmark} & \textbf{Difficulty label} & \textbf{Items} & \textbf{Clusters} & \textbf{Items } \\
\midrule
MATH-500 & level 1-5 & 500 & 500 & 43/90/105/128/134  \\
ProofWriter & proof depth 1-5 & 500 & 100 & 100 each \ \\
StepGame & reasoning hops 1-5 & 500 & 500 & 100 each  \\
SpaRTUN & reasoning hops 1-3 & 610 & 433 & 268/308/34  \\
GSM8K $\to$ MATH-500 $\to$ OlympiadBench & benchmark tier & 150 & 150 & 50 each  \\
\bottomrule
\end{tabular}}
\caption{Difficulty benchmarks. Clusters are the resampling units of the cluster bootstrap (theories for ProofWriter, stories for SpaRTUN, items elsewhere).}
\label{tab:difficulty-data}
\end{table}

\paragraph{Readout Window Comparison}

Table~\ref{tab:windows} compares four fixed scalar readouts of the BTE depth profile: early (layers 0–7), middle (12–19; used in the main), late (24–31), and the full-depth mean. The middle window is the only fixed readout that remains positively associated with difficulty across all five settings, with its 95\% confidence interval excluding zero in every case. Notably, it is not always the strongest readout, the full-depth mean performs best in the cross-benchmark mathematical comparison, but it provides the most consistent direction of association across tasks.

\begin{table}[htbp]
\centering
\adjustbox{max width=\linewidth}{%
\small
\begin{tabular}{lcccc}
\toprule
\textbf{Benchmark} & \textbf{Early (L0-7)} & \textbf{Middle (L12-19)} & \textbf{Late (L24-31)} & \textbf{Full (L0-31)} \\
\midrule
MATH-500 & 0.18 [0.09, 0.26] & \textbf{0.44} [0.37, 0.51] & 0.18 [0.10, 0.26] & 0.33 [0.25, 0.41] \\
ProofWriter & 0.27 [0.21, 0.34] & \textbf{0.33} [0.26, 0.40] & $-$0.45 [$-$0.52, $-$0.38] & $-$0.18 [$-$0.26, $-$0.11] \\
StepGame & 0.59 [0.52, 0.66] & \textbf{0.65} [0.58, 0.70] & $-$0.91 [$-$0.92, $-$0.89] & $-$0.15 [$-$0.25, $-$0.06] \\
SpaRTUN & 0.11 [0.03, 0.19] & \textbf{0.66} [0.62, 0.70] & $-$0.33 [$-$0.39, $-$0.26] & 0.07 [0.00, 0.15] \\
GSM8K$\to$MATH$\to$Olympiad & 0.77 [0.69, 0.82] & 0.44 [0.30, 0.57] & 0.77 [0.69, 0.84] & \textbf{0.79} [0.71, 0.84] \\
\bottomrule
\end{tabular}}
\caption{Item-level Spearman $\rho$ between the readout and the difficulty label (Llama-3.1-8B scorer, reference solutions), with 95\% cluster-bootstrap CIs. Bold: largest $\rho$ per row. The middle window is the only readout that is positive on every benchmark.}
\label{tab:windows}
\end{table}

\paragraph{Writer control.}
Official reference solutions differ in style across the three mathematics benchmarks: GSM8K solutions are short, informal step lists, whereas OlympiadBench solutions are competition-style proofs. To check that the cross-task trend is not an artefact of benchmark-specific writing, we had GPT-5.5 rewrite all 150 solutions under one instruction and re-scored them with Llama (Table~\ref{tab:writer-control}). The ordering GSM8K $<$ MATH-500 $<$ OlympiadBench is unchanged and the middle-layer trend becomes stronger ($\rho = 0.44 \to 0.75$).

\begin{table}[htbp]
\centering
\small
\begin{tabular}{lcc}
\toprule
\textbf{Solution source} & $\boldsymbol{n}$ & $\boldsymbol{\rho}$ \textbf{[95\% CI]} \\
\midrule
Official references & 150 & 0.44 [0.29, 0.57] \\
GPT-5.5 rewrites    & 150 & 0.75 [0.68, 0.81] \\
\bottomrule
\end{tabular}
\caption{\textbf{Writer control for cross-task mathematical difficulty.}
Spearman correlation between middle-layer BTE and the benchmark ordering
GSM8K $\rightarrow$ MATH-500 $\rightarrow$ OlympiadBench, using either
official reference solutions or GPT-5.5 rewrites under a shared instruction.}
\label{tab:writer-control}
\end{table}

\subsection{Review Data and Detector Details}
\label{app:review-data}

\paragraph{Data and splits.}
We use ICLR and NeurIPS reviews from 2021-2024. Each conference-year contributes 200 human-generated review pairs per generator, for 11,200 pairs across seven generators and eight conference-years; pairs for different generators may share the human review. GPT-4o~\citep{gpt4o}, Gemini-1.5-Pro~\citep{gemini1.5}, and Claude-3.5-Sonnet reviews are from~\citet{yu2025your}; we generated the GPT-5.5~\citep{gpt5}, GPT-5.6, Gemini-3.1-Pro~\citep{gemini31pro}, and Claude-Opus-5~\citep{opus5} reviews for the same papers. All splits are by paper, and shared human reviews are counted only once when fitting the classifier. Table~\ref{tab:review-data} gives the training sizes for the three settings. The common test set contains 1,033 review pairs per generator from ICLR 2022-2024 and NeurIPS 2021-2024. In the Rewrite arm, GPT-5.5 is instructed to preserve the human review's content and structure while editing its wording.  For BTE scoring, we convert each paper PDF to Markdown with MinerU, place the parsed paper text in the context prefix, and calculate the score on the corresponding review position.
\begin{table}[htbp]
\centering
\small
\adjustbox{max width=\linewidth}{%
\begin{tabular}{@{}lcc@{}}
\toprule
\textbf{Setting} & \textbf{Training pairs} & \textbf{Test pairs / generator} \\
\midrule
A: three legacy generators & 162 & 1,033 \\
B: leave one family out & 216 / 270 & 1,033 \\
C: all generators and years & 2,947 & 1,033 \\
\bottomrule
\end{tabular}
}
\caption{Human-generated review-pair counts for the paper-disjoint settings. A pairs ICLR 2021 reviews with GPT-4o, Gemini-1.5-Pro, and Claude-3.5-Sonnet. B fits a separate detector for each held-out family: 216 training pairs when GPT is held out and 270 when Gemini or Claude is held out. C uses all seven generators and papers from all eight conference-years.}
\label{tab:review-data}
\end{table}

\paragraph{BTE representation and readout.}
For each review, we use its BTE depth profile $\mathbf{e}=(e_0,\ldots,e_{L-1})$ as the \emph{full} representation. To isolate the relative shape of the profile from its overall magnitude, we additionally consider a \emph{shape-only} representation obtained by normalizing each review across depth:
\begin{equation}
\tilde e_\ell
=
\frac{e_\ell-\mu_{\mathbf e}}
{\sigma_{\mathbf e}+\varepsilon},
\qquad \ell=0,\ldots,L-1,
\label{eq:shape_normalization}
\end{equation}
where $\mu_{\mathbf e}$ and $\sigma_{\mathbf e}$ denote the mean and standard deviation of $\mathbf e$ across layers, and $\varepsilon$ is a small constant for numerical stability. This normalization removes review-level differences in BTE magnitude and scale while preserving the relative pattern of predictive revisions across depth.

For each scoring model, representation (\emph{full} or \emph{shape-only}), and training setting, we fit a class-balanced logistic regression. Features are standardized using statistics computed from the training partition only, and the regularization strength is selected by five-fold cross-validation grouped by paper. LLM-generated reviews are treated as the positive class. Rewritten reviews are used only for evaluation and are never included when fitting the classifier.

\paragraph{Training settings and evaluation.}
In A, the three legacy generators above are the only training generators, so GPT-5.5, GPT-5.6, Gemini-3.1-Pro, and Claude-Opus-5 are unseen. In B, we hold out an entire generator family in turn (GPT, Gemini, or Claude), fit on the remaining families using the same ICLR 2021 papers on the main page (with NeurIPS 2021 as the source in Table~\ref{tab:bte-review-neurips2021-ab}), and score each test generator with the detector that did not see its family. In C, all seven generators and all eight conference-years are represented in training, while test papers remain disjoint. We compute AUROC separately for each of the seven test conference-years and seven generators, then average the 49 values equally. A and B train on ICLR 2021 only; C also increases the amount of training data, so their differences do not isolate a single factor.

\paragraph{Reference detectors.}
Following the recent benchmark study of~\citet{yu2025your}, we select a representative subset of the strongest-performing detectors in their evaluation, balancing baseline coverage with computational cost.
All surrogate-model scores are computed on exactly the BTE inputs: the full review with the full paper as prefix, under the same scorer. Log-likelihood detection~\citep{solaiman2019releasestrategiessocialimpacts} uses the mean next-token log-probability of the review under Llama-3.1-8B-Instruct or Qwen-3.5-9B. SpecDetect~\citep{luo2026specdetect} uses the negative Fourier energy of the centered token log-probabilities of Llama-3.1-8B-Instruct (SpecDetect's original proxy is GPT-J); our additional \emph{normalized} variant divides this energy by the squared number of scored tokens. Binoculars~\citep{hans2024spotting} uses Llama-3.1-8B (base) as observer and Llama-3.1-8B-Instruct as performer, so that the observer/performer pair sees the same context as BTE; the published Falcon-7B pair cannot take the paper prefix within its 2k context. Our transductive Open-Anchor variant~\citep{yu2025your} uses Qwen3-Embedding-8B and three paper-matched anchors from GPT-4o, Gemini-1.5, and Claude-3.5; its score is the maximum cosine similarity after excluding the tested generator's own anchor. Anchor-source reviews are also excluded, leaving 713 test pairs per generator rather than 1,033, so its row is not measured on the identical review set.

\begin{table}[!htbp]
\centering
\setlength{\tabcolsep}{1.15pt}
\renewcommand{\arraystretch}{1.08}
\fontsize{7.2}{8.5}\selectfont
\begin{adjustbox}{max width=\linewidth}
\begin{tabular}{@{}l*{16}{c}@{}}
\toprule
\multirow{2}{*}{\textbf{Method}} & \multicolumn{8}{c}{\textbf{Original human reviews}} & \multicolumn{8}{c}{\textbf{GPT-5.5 rewrite of human reviews}} \\
\cmidrule(lr){2-9}\cmidrule(l){10-17}
& \makecell{GPT\\4o} & \makecell{Gem\\1.5} & \makecell{Clau\\3.5} & \makecell{GPT\\5.5} & \makecell{GPT\\5.6} & \makecell{Gem\\3.1} & \makecell{Opus\\5} & \textbf{Avg.}
& \makecell{GPT\\4o} & \makecell{Gem\\1.5} & \makecell{Clau\\3.5} & \makecell{GPT\\5.5} & \makecell{GPT\\5.6} & \makecell{Gem\\3.1} & \makecell{Opus\\5} & \textbf{Avg.} \\
\midrule
\multicolumn{17}{@{}l}{\textit{Setting A: three legacy training generators; seven held-out conference-years}} \\
\quad Llama-3.1-8B (shape-only, 32D) & 0.9897 & 0.9917 & 0.9899 & 0.3609 & 0.1700 & 0.9453 & 0.0902 & 0.6482 & 0.9876 & 0.9903 & 0.9882 & 0.3033 & 0.1288 & 0.9292 & 0.0622 & 0.6271 \\
\rowcolor{gray!10}
\quad Llama-3.1-8B (full, 32D) & 0.9914 & 0.9835 & 0.9798 & 0.7124 & 0.5877 & 0.9088 & 0.1348 & 0.7569 & 0.9899 & 0.9808 & 0.9759 & 0.6699 & 0.5379 & 0.8903 & 0.1085 & 0.7362 \\
\quad Qwen-3.5-9B (shape-only, 32D) & 0.9951 & 0.9885 & 0.9943 & 0.9369 & 0.8168 & 0.9584 & 0.7675 & 0.9225 & 0.9952 & 0.9887 & 0.9945 & 0.9372 & 0.8199 & 0.9590 & 0.7697 & 0.9235 \\
\rowcolor{gray!10}
\quad Qwen-3.5-9B (full, 32D) & \textbf{0.9960} & 0.9899 & 0.9900 & 0.8782 & 0.7603 & 0.9631 & 0.7782 & 0.9080 & \textbf{0.9957} & 0.9888 & 0.9886 & 0.8475 & 0.7020 & 0.9553 & 0.7210 & 0.8855 \\
\midrule
\multicolumn{17}{@{}l}{\textit{Setting B: leave-one-family-out training; seven held-out conference-years}} \\
\quad Llama-3.1-8B (shape-only, 32D) & 0.9011 & 0.8665 & 0.9758 & 0.8690 & 0.7672 & 0.9397 & 0.7850 & 0.8720 & 0.8955 & 0.8710 & 0.9768 & 0.8596 & 0.7537 & 0.9410 & 0.7874 & 0.8693 \\
\rowcolor{gray!10}
\quad Llama-3.1-8B (full, 32D) & 0.9285 & 0.9327 & 0.9654 & 0.9052 & 0.8829 & 0.9219 & 0.6906 & 0.8896 & 0.9192 & 0.9288 & 0.9621 & 0.8918 & 0.8673 & 0.9161 & 0.6678 & 0.8790 \\
\quad Qwen-3.5-9B (shape-only, 32D) & 0.9820 & \textbf{0.9955} & \textbf{0.9977} & \textbf{0.9974} & \textbf{0.9948} & \textbf{0.9979} & \textbf{0.9934} & \textbf{0.9941} & 0.9810 & \textbf{0.9951} & \textbf{0.9978} & \textbf{0.9975} & \textbf{0.9950} & \textbf{0.9979} & \textbf{0.9930} & \textbf{0.9939} \\
\rowcolor{gray!10}
\quad Qwen-3.5-9B (full, 32D) & 0.9954 & 0.9947 & 0.9962 & 0.9697 & 0.9485 & 0.9942 & 0.9907 & 0.9842 & 0.9951 & 0.9941 & 0.9960 & 0.9640 & 0.9372 & 0.9935 & 0.9895 & 0.9814 \\
\bottomrule
\end{tabular}
\end{adjustbox}
\normalsize
\caption{NeurIPS 2021 source-year transfer for human-LLM review detection. }
\label{tab:bte-review-neurips2021-ab}
\end{table}

\paragraph{Scorer Ablation}
\label{app:scorer-ablation}

Figure~\ref{fig:scorer-era} repeats the review-only detection experiment with five scorers: two Llama-architecture models that differ only in post-training (Llama-3.1-8B-Instruct and its R1-distilled variant), and three generations of the Qwen family (Qwen2.5-7B-Instruction, Qwen3-8B, Qwen3.5-9B).
Table~\ref{tab:familiarity} tests the familiarity account directly. For each scorer we compute the mean next-token negative log-likelihood (NLL) of human and generated reviews of the same papers. Every scorer finds generated reviews easier to predict than human reviews, and the size of this gap grows across the Qwen generations (5.0\%, 6.4\%, 10.4\%) while it is only 2.4\% for Llama.

\begin{table}[htbp]
\centering
\small
\adjustbox{max width=\linewidth}{%
\setlength{\tabcolsep}{4pt}
\begin{tabular}{lcccccc}
\toprule
\textbf{Scorer} & \textbf{NLL human} & \textbf{NLL generated} & \textbf{Gap} & \textbf{Mid-depth BTE deficit} & \textbf{B} & \textbf{C} \\
\midrule
Llama-3.1-8B-Instruct (2024)     & 2.716 & 2.652 & $+2.4\%$  & $-9\%$   & 0.938 & 0.979 \\
DeepSeek-R1-Distill-Llama-8B (2025) & 3.363 & 3.261 & $+3.0\%$  & $-1\%$   & 0.873 & 0.958 \\
Qwen2.5-7B-Instruct (2024)       & 2.799 & 2.660 & $+5.0\%$  & $-10\%$  & 0.970 & 0.985 \\
Qwen3-8B (2025)                  & 2.885 & 2.700 & $+6.4\%$  & $-11\%$  & 0.955 & 0.985 \\
Qwen3.5-9B (2026)                & 2.477 & 2.219 & $+10.4\%$ & $-13\%$  & 0.992 & 0.997 \\
\bottomrule
\end{tabular}}
\caption{Familiarity of each scorer with generated versus human reviews. NLL is the mean final-layer next-token negative log-likelihood (nats) over 100 human-generated review pairs per generator from ICLR 2022 and NeurIPS 2023, scored without the paper; Gap is the relative reduction for generated reviews. Mid-depth BTE deficit is the mean energy of legacy-generator reviews relative to human reviews over the middle third of layers on the full corpus. B and C are the macro-AUROC of the full-32D detector in the review-only setting (Figure~\ref{fig:scorer-era}).}
\label{tab:familiarity}
\end{table}

\begin{figure}[tbp]
    \centering
    \includegraphics[
        width=\textwidth,
        keepaspectratio
    ]{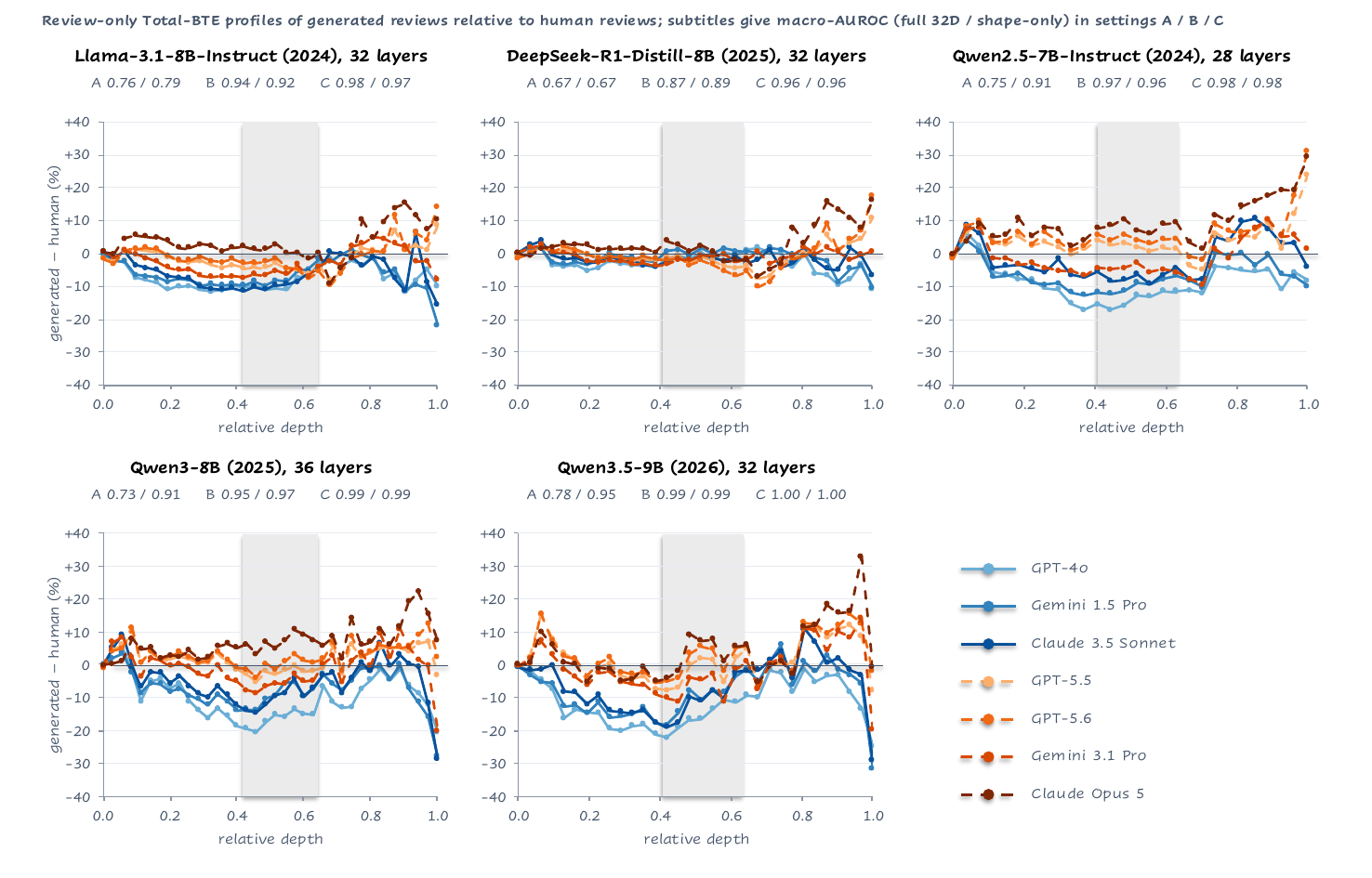}
\caption{Review-only depth profiles of generated reviews relative to human reviews under five scorers, with layer index normalised to $[0,1]$. Subtitles give macro-AUROC (full 32D / shape-only) in settings A / B / C. Solid: legacy generators; dashed: newer generators.}
\label{fig:scorer-era}
\end{figure}

\paragraph{Per-Year Breakdown of Review Detection}
\label{app:year-ablation}

Table~\ref{tab:year-macro} splits the macro-AUROC of Table~\ref{tab:bte-review-three-settings} by test conference-year (Original human reviews, full 32D features). Three points stand out. First, venue does not matter: ICLR 2022-2024 and NeurIPS 2021-2024 average within 0.01 of each other in settings B and C, so training on ICLR 2021 gives no advantage on later ICLR years. Second, the most recent year is the hardest in every setting and for both scorers: NeurIPS 2024 is the minimum of every column, and ICLR 2024 is the second lowest for Qwen. Third, exposure to more generators and years shrinks the spread across years, from 0.18 (A, Llama) to 0.05 (C, Llama) and 0.007 (C, Qwen). Table~\ref{tab:year-gen-A} shows that in setting~A the three seen generators stay above 0.94 in every year, while the failure on the unseen newer generators appears in every year rather than in a particular one, and is most severe in 2024. \textbf{Since the human reviews of 2024 may themselves have been written with LLM assistance, the 2024 cells should be read as an upper bound on the difficulty of that year rather than as a property of the detector.}

\begin{table}[tbp]
\centering
\small
\begin{tabular}{lcccccc}
\toprule
& \multicolumn{2}{c}{\textbf{A}} & \multicolumn{2}{c}{\textbf{B}} & \multicolumn{2}{c}{\textbf{C}} \\
\cmidrule(lr){2-3}\cmidrule(lr){4-5}\cmidrule(lr){6-7}
\textbf{Test conference-year} & Llama & Qwen & Llama & Qwen & Llama & Qwen \\
\midrule
ICLR 2022    & 0.8129 & 0.9314 & 0.8939 & 0.9997 & 0.9601 & 1.0000 \\
ICLR 2023    & 0.7556 & 0.9357 & 0.9441 & 0.9999 & 0.9827 & 0.9998 \\
ICLR 2024    & 0.7723 & 0.8729 & 0.8966 & 0.9865 & 0.9616 & 0.9932 \\
NeurIPS 2021 & 0.7430 & 0.9282 & 0.9444 & 1.0000 & 0.9822 & 1.0000 \\
NeurIPS 2022 & 0.7450 & 0.9231 & 0.9389 & 1.0000 & 0.9872 & 0.9999 \\
NeurIPS 2023 & 0.7247 & 0.8727 & 0.9228 & 0.9936 & 0.9605 & 0.9997 \\
NeurIPS 2024 & 0.6375 & 0.7735 & 0.8250 & 0.9781 & 0.9420 & 0.9934 \\
\midrule
ICLR mean    & 0.7803 & 0.9133 & 0.9115 & 0.9954 & 0.9681 & 0.9976 \\
NeurIPS mean & 0.7126 & 0.8744 & 0.9078 & 0.9929 & 0.9680 & 0.9982 \\
\bottomrule
\end{tabular}
\caption{AUROC by test conference-year, averaged over the seven generators (Original human reviews, full 32D). Settings A-C and scorers as in Table~\ref{tab:bte-review-three-settings}.}
\label{tab:year-macro}
\end{table}

\begin{table}[!htb]
\centering
\small
\setlength{\tabcolsep}{4pt}
\begin{tabular}{lccccccc}
\toprule
\textbf{Test conference-year} & GPT-4o & Gem-1.5 & Clau-3.5 & GPT-5.5 & GPT-5.6 & Gem-3.1 & Opus-5 \\
\midrule
\multicolumn{8}{l}{\textit{Llama-3.1-8B, setting A}} \\
ICLR 2022    & 0.999 & 0.999 & 0.999 & 0.937 & 0.860 & 0.792 & 0.105 \\
ICLR 2023    & 1.000 & 0.999 & 0.998 & 0.689 & 0.498 & 0.958 & 0.147 \\
ICLR 2024    & 0.990 & 0.968 & 0.939 & 0.788 & 0.738 & 0.895 & 0.088 \\
NeurIPS 2021 & 1.000 & 0.998 & 0.999 & 0.649 & 0.493 & 0.943 & 0.120 \\
NeurIPS 2022 & 1.000 & 0.998 & 0.996 & 0.639 & 0.510 & 0.934 & 0.139 \\
NeurIPS 2023 & 0.985 & 0.983 & 0.981 & 0.626 & 0.427 & 0.952 & 0.118 \\
NeurIPS 2024 & 0.970 & 0.947 & 0.958 & 0.416 & 0.272 & 0.855 & 0.044 \\
\midrule
\multicolumn{8}{l}{\textit{Qwen-3.5-9B, setting A}} \\
ICLR 2022    & 1.000 & 1.000 & 1.000 & 0.962 & 0.846 & 0.972 & 0.740 \\
ICLR 2023    & 1.000 & 1.000 & 1.000 & 0.905 & 0.791 & 0.995 & 0.858 \\
ICLR 2024    & 0.993 & 0.984 & 0.961 & 0.864 & 0.707 & 0.950 & 0.652 \\
NeurIPS 2021 & 1.000 & 1.000 & 1.000 & 0.932 & 0.801 & 1.000 & 0.764 \\
NeurIPS 2022 & 1.000 & 1.000 & 1.000 & 0.893 & 0.799 & 0.995 & 0.775 \\
NeurIPS 2023 & 0.988 & 0.985 & 0.986 & 0.820 & 0.659 & 0.946 & 0.725 \\
NeurIPS 2024 & 0.993 & 0.961 & 0.978 & 0.632 & 0.474 & 0.868 & 0.508 \\
\bottomrule
\end{tabular}
\caption{Setting A (three legacy training generators, ICLR 2021 papers) by test conference-year and generator, full 32D, Original human reviews. The first three generators are seen in training.}
\label{tab:year-gen-A}
\end{table}

\subsection{Human Evaluation}
\label{app:human-eval}

\paragraph{Evaluators and materials.}
We recruited six evaluators, all NLP researchers. The evaluation set contains 30 language-focused papers from ICLR and NeurIPS 2021-2024 (15 papers from each venue), including NLP, language-model reasoning, speech, and language-related multimodal work. For each paper, we collected one human review and one review from each of the seven generators, yielding 240 reviews in total. Paper selection used only venue, year, and frozen topic metadata.

\paragraph{Assignment and blinding.}
Each evaluator was assigned ten papers (five from ICLR and five from NeurIPS) and compared ten review pairs per paper. The ten pairs comprised all seven human-model comparisons and three model-model comparisons. Every paper and every pair was evaluated independently by two evaluators. The paired evaluators saw the same reviews in opposite A/B orientations, while paper order and within-paper comparison order were pseudorandomly shuffled. Review provenance, model identity, and BTE scores were hidden. This design produced 300 distinct pairs and 600 completed judgments: 420 judgments on 210 human-model pairs and 180 judgments on 90 model-model pairs. Each generator was therefore compared with the matched human review on all 30 papers, with two judgments per paper ($n=60$). All six evaluators completed their 100 assigned comparisons; no evaluator was excluded.

\begin{tcolorbox}[
  enhanced,
  breakable=false,
  before skip=4pt,
  after skip=4pt,
  colback=black!2,
  colframe=black!38,
  colbacktitle=black!7,
  coltitle=black,
  boxrule=0.55pt,
  arc=1.2mm,
  left=1.8mm,
  right=1.8mm,
  top=1.2mm,
  bottom=1.2mm,
  fonttitle=\bfseries,
  title={Evaluator guidelines (abridged)}
]
\fontsize{8.5}{9.3}\selectfont
\setlength{\parskip}{2pt}
\titlespacing*{\paragraph}{0pt}{2pt}{0.7em}
\paragraph{Use the paper as evidence.} Read the assigned paper before comparing the two reviews. Judge only from the paper and review text; do not consult public reviews, rebuttals, or discussion. If the paper is inaccessible, its version is unclear, or the content cannot be assessed reliably, pause and contact the study coordinator.

\paragraph{Compare quality without inferring provenance.} Reviews A and B are anonymous. Do not guess whether either review was written by a human or an LLM. Select \emph{A better} or \emph{B better} only when one review has a substantive advantage supported by the paper. Select \emph{Materially similar} when the difference is small or the advantages offset one another; do not decide from length, tone, or stylistic polish alone.

\paragraph{Separate grounding from depth.} \emph{Grounding} asks whether a review's claims are accurate, specific, and supported by the paper. \emph{Depth} asks whether it identifies the central issue and explains why an observation matters, including its causes and consequences for the paper's claims, rather than merely restating it. Improvement suggestions may strengthen a criticism but are not required for a weakness to be deep.

\paragraph{Handle questions conditionally.} A substantive question may appear in a dedicated Questions section or within Weaknesses and must request clarification or additional evidence; merely stating a weakness is not a question. Rate the two Questions dimensions only when both reviews contain substantive questions. Otherwise record which side is missing such questions and skip both Questions items.
\end{tcolorbox}

\paragraph{Rubric and scoring.}
For each pair, evaluators selected \emph{A better}, \emph{B better}, or \emph{Materially similar} on nine dimensions: (i) overall scientific usefulness, quality, and depth; grounding and depth for (ii-iii) Summary, (iv-v) Strengths, and (vi-vii) Weaknesses; and grounding and depth/usefulness for (viii-ix) Questions. The Overall judgment is our primary human-evaluation outcome. The remaining dimensions are reported separately and are not summed into a composite score. The two Questions dimensions are defined only for judgments in which both reviews contain substantive questions.
After restoring the hidden source orientation, a model win, a materially similar judgment, and a human win receive scores of $1$, $0.5$, and $0$, respectively. For generator $g$, the tie-adjusted preference score is
\begin{equation}
P_g=\frac{W_g+0.5T_g}{W_g+T_g+H_g},
\label{eq:human-preference}
\end{equation}
where $W_g$, $T_g$, and $H_g$ denote model wins, ties, and human wins. Thus, $P_g=0.5$ denotes parity with the matched human reviews. We obtain 95\% confidence intervals by resampling the 30 papers while retaining both judgments for each sampled paper. Papers, rather than the 60 individual judgments, are the primary sampling units.

\begin{table}[htbp]
\centering
\small
\setlength{\tabcolsep}{4.0pt}
\begin{tabular}{@{}lrrrcc@{}}
\toprule
\textbf{Generator} & $\boldsymbol{W}$ & $\boldsymbol{T}$ & $\boldsymbol{H}$ & $\boldsymbol{P_g}$ & \textbf{95\% CI} \\
\midrule
GPT-4o                & 23 & 6 & 31 & 0.433 & $[0.300,0.575]$ \\
Gemini~1.5~Pro        & 38 & 2 & 20 & 0.650 & $[0.508,0.783]$ \\
Claude~3.5~Sonnet     & 29 & 3 & 28 & 0.508 & $[0.367,0.650]$ \\
GPT-5.5               & 49 & 6 &  5 & 0.867 & $[0.775,0.942]$ \\
GPT-5.6               & 51 & 5 &  4 & 0.892 & $[0.800,0.967]$ \\
Gemini~3.1~Pro        & 28 & 2 & 30 & 0.483 & $[0.333,0.633]$ \\
Claude~Opus~5         & 50 & 5 &  5 & 0.875 & $[0.775,0.958]$ \\
\bottomrule
\end{tabular}
\caption{Human preference results on the primary Overall dimension. $W$, $T$, and $H$ denote model wins, materially similar judgments, and human wins. Each row contains 60 judgments over the same 30 papers, with two evaluators per paper. Scores above $0.5$ favor the generated review. Confidence intervals use a paper-clustered bootstrap.}
\label{tab:human-eval-overall}
\end{table}

\paragraph{Primary results and agreement.}
GPT-5.5, GPT-5.6, and Claude~Opus~5 each substantially exceed parity, with respective preference scores of $0.867$, $0.892$, and $0.875$. Collectively, they were preferred to the matched human review in 150 of 180 judgments (83.3\%), with 16 ties (8.9\%) and 14 human wins (7.8\%). Their pooled preference score is $0.878$ (paper-bootstrap 95\% CI: $[0.797,0.939]$), compared with $0.531$ ($[0.408,0.653]$) for GPT-4o, Gemini~1.5, and Claude~3.5. An exploratory contrast between these observed high-performing and legacy groups gives a difference of $0.347$ (paper-bootstrap 95\% CI: $[0.236,0.458]$; two-way paper- and evaluator-clustered $t_5=5.87$, $p=0.002$). Gemini~3.1 is an exception among the newer generators, with $P_g=0.483$. The pooled score of the three high-performing generators remains between $0.853$ and $0.913$ when any one evaluator is removed.
After mapping both raters' choices back to the same source orientation, exact agreement on Overall was 71.0\% for the primary human-model pairs; ordinal Krippendorff's $\alpha$ was $0.496$. Across all 300 pairs, exact Overall agreement was 68.7\%.

\begin{table}[tbp]
\centering
\small
\setlength{\tabcolsep}{5.0pt}
\renewcommand{\arraystretch}{1.08}
\adjustbox{max width=\linewidth}{%
\begin{tabular}{@{}lcccccccccc@{}}
\toprule
\multirow{2}{*}{\textbf{Generator}}
& \multirow{2}{*}{\textbf{Overall}}
& \multicolumn{2}{c}{\textbf{Summary}}
& \multicolumn{2}{c}{\textbf{Strengths}}
& \multicolumn{2}{c}{\textbf{Weaknesses}}
& \multicolumn{2}{c}{\textbf{Questions}}
& \multirow{2}{*}{$\boldsymbol{n_Q}$} \\
\cmidrule(lr){3-4}\cmidrule(lr){5-6}\cmidrule(lr){7-8}\cmidrule(lr){9-10}
& & \textbf{G} & \textbf{D} & \textbf{G} & \textbf{D} & \textbf{G} & \textbf{D} & \textbf{G} & \textbf{D} & \\
\midrule
GPT-4o            & 0.433 & 0.700 & 0.650 & 0.725 & 0.683 & 0.383 & 0.392 & 0.451 & 0.390 & 41 \\
Gemini~1.5~Pro    & 0.650 & 0.775 & 0.775 & 0.800 & 0.742 & 0.525 & 0.550 & 0.705 & 0.682 & 44 \\
Claude~3.5~Sonnet & 0.508 & 0.617 & 0.583 & 0.742 & 0.733 & 0.458 & 0.425 & 0.513 & 0.538 & 40 \\
GPT-5.5           & 0.867 & 0.675 & 0.717 & 0.808 & 0.792 & 0.775 & 0.792 & 0.849 & 0.826 & 43 \\
GPT-5.6           & 0.892 & 0.742 & 0.717 & 0.833 & 0.833 & 0.792 & 0.792 & 0.837 & 0.804 & 46 \\
Gemini~3.1~Pro    & 0.483 & 0.550 & 0.500 & 0.592 & 0.608 & 0.400 & 0.475 & 0.513 & 0.436 & 39 \\
Claude~Opus~5     & 0.875 & 0.733 & 0.758 & 0.842 & 0.867 & 0.808 & 0.817 & 0.841 & 0.854 & 41 \\
\bottomrule
\end{tabular}%
}
\caption{Tie-adjusted model preference scores by evaluation dimension. G and D denote grounding and depth, respectively; for Questions, D denotes depth. Overall, Summary, Strengths, and Weaknesses use 60 judgments per generator. Questions are evaluated only when the evaluator judged both reviews to contain substantive questions, and $n_Q$ gives the resulting number of judgments.}
\label{tab:human-eval-dimensions}
\end{table}

\paragraph{Association with BTE.}
We next test whether the generator differences observed by human evaluators are reflected in their BTE profiles. Let $e^{(s)}_{g,p,\ell}$ denote the response-token-averaged BTE at layer transition $\ell$ for the review generated by model $g$ for paper $p$, as measured by scorer $s$. Let $e^{(s)}_{\mathrm{human},p,\ell}$ denote the corresponding value for the matched human review of the same paper. For each generator, we compute the signed middle-layer BTE shift
\begin{equation}
\Delta_g=\frac{1}{2}\sum_{s\in\{\mathrm{Llama},\mathrm{Qwen}\}}
\frac{1}{30}\sum_{p=1}^{30}\frac{1}{8}\sum_{\ell=12}^{19}
\left(e^{(s)}_{g,p,\ell}-e^{(s)}_{\mathrm{human},p,\ell}\right).
\label{eq:human-bte-shift}
\end{equation}
Across the seven generators, $\Delta_g$ correlates with the Overall preference score at Spearman's $\rho=0.857$. The association is similar when the two scorers are analyzed separately (Llama: $\rho=0.786$, $p=0.048$; Qwen: $\rho=0.857$, $p=0.024$), and the ranking induced by $\Delta_g$ is unchanged when BTE means are computed over the complete eight-conference corpus.

\subsection{Attribution: Feature Sets and Transfer Settings}
\label{app:attribution-transfer}

\begin{figure}[tbp]
    \centering
    \includegraphics[
        width=\textwidth,
        keepaspectratio
    ]{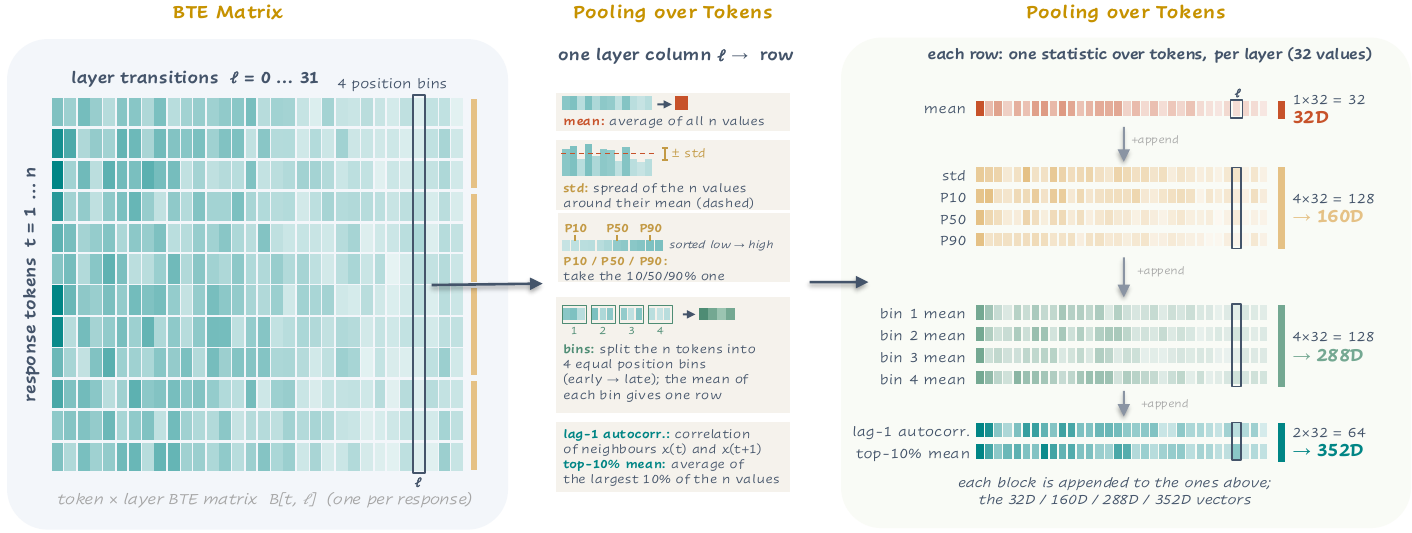}
\caption{Attribution features. Each layer column of the token~$\times$~layer BTE matrix (left)
is summarised over tokens (middle) and the resulting per-layer statistics are concatenated
into 32D-352D feature vectors (right).}
\label{fig:feature-illustration}
\end{figure}

All feature sets are computed from the same matrix of token-level Total BTE (one row per response token, one column per layer). \emph{Token mean} averages each column (32D). Each further row of Table~\ref{tab:attribution-features} adds, per layer, the standard deviation and the 10th/50th/90th percentiles over tokens (160D), the mean within four equal position bins (288D), and the lag-1 autocorrelation plus the mean of the 10\% highest-energy tokens (352D). We also test the position bins alone (128D) (see Figure~\ref{fig:feature-illustration}). The classifier is a class-balanced multinomial logistic regression with $C=1$ on standardised features over eight classes (Human + seven generators); a generated review predicted as Human counts as an error.  \textbf{A} trains on ICLR 2021 and tests on the other seven conference-years (9,800 generated reviews); \textbf{B} trains on ICLR 2021-2024 and tests on NeurIPS 2021-2024 (5,600); \textbf{C}, reported in the main text, trains on the calibration papers of all eight conference-years and tests on the remaining papers (8,253). Table~\ref{tab:attribution-transfer} gives all results. The token distribution contributes most in every setting ($+5$ to $+6$ points for Llama, $+3$ to $+4$ for Qwen).

\begin{table}[htbp]
\centering
\small
\setlength{\tabcolsep}{4pt}
\begin{tabular}{lr ccc ccc}
\toprule
& & \multicolumn{3}{c}{\textbf{Llama-3.1-8B-Instruction}} & \multicolumn{3}{c}{\textbf{Qwen-3.5-9B}} \\
\cmidrule(lr){3-5}\cmidrule(lr){6-8}
\textbf{Per-layer features} & \textbf{Dim.} & A & B & C & A & B & C \\
\midrule
Token mean                          &  32 & 72.1 & 82.9 & 86.4 & 82.9 & 89.5 & 91.3 \\
\quad + std, P10/P50/P90            & 160 & \textbf{77.9} & \textbf{88.5} & 91.8 & \textbf{87.1} & 93.1 & 94.4 \\
\quad + means in 4 position bins    & 288 & 76.4 & 86.2 & 92.7 & 87.0 & 93.2 & 95.0 \\
\quad + AC, top-10\% mean           & 352 & 77.1 & 87.1 & \textbf{93.2} & 86.6 & \textbf{94.0} & \textbf{95.6} \\
Position bins only                  & 128 & 70.5 & 82.1 & 88.6 & 79.6 & 90.1 & 92.6 \\
\bottomrule
\end{tabular}
\caption{Eight-way source accuracy (\%) on generated reviews under settings A-C.}
\label{tab:attribution-transfer}
\end{table}

\begin{table}[htbp]
\centering
\small
\setlength{\tabcolsep}{3pt}
\begin{tabular}{l ccc ccc}
\toprule
& \multicolumn{3}{c}{\textbf{Llama-3.1-8B-Instruct}} & \multicolumn{3}{c}{\textbf{Qwen-3.5-9B}} \\
\cmidrule(lr){2-4}\cmidrule(lr){5-7}
\textbf{Per-layer features} & A & B & C & A & B & C \\
\midrule
Token mean                        & 79.1 / 2.3 & 74.2 / 0.2 & 85.6 / 1.2 & 93.9 / 0.0 & 97.1 / 0.1 & 98.6 / 0.3 \\
\quad + std, P10/P50/P90          & 85.1 / 0.7 & 94.8 / 0.5 & 93.3 / 0.8 & 94.1 / 0.0 & 97.8 / 0.1 & 98.3 / 0.1 \\
\quad + means in 4 position bins  & 86.7 / 0.8 & 94.6 / 0.5 & 93.6 / 0.5 & 94.9 / 0.1 & 98.2 / 0.1 & 98.0 / 0.0 \\
\quad + AC, top-10\% mean         & 85.6 / 0.7 & 95.4 / 0.3 & 94.1 / 0.5 & 94.4 / 0.0 & 98.1 / 0.0 & 98.0 / 0.1 \\
Position bins only                & 84.3 / 2.3 & 89.2 / 0.5 & 90.2 / 1.0 & 94.9 / 0.1 & 97.4 / 0.2 & 98.7 / 0.2 \\
\bottomrule
\end{tabular}
\caption{Human-versus-machine behaviour of the same eight-way classifiers as Table~\ref{tab:attribution-transfer}: human recall (\%) / share of generated reviews assigned to the Human class (\%).}
\label{tab:attribution-human}
\end{table}

\begin{figure}[tbp]
    \centering
    \includegraphics[
        width=\textwidth,
        keepaspectratio
    ]{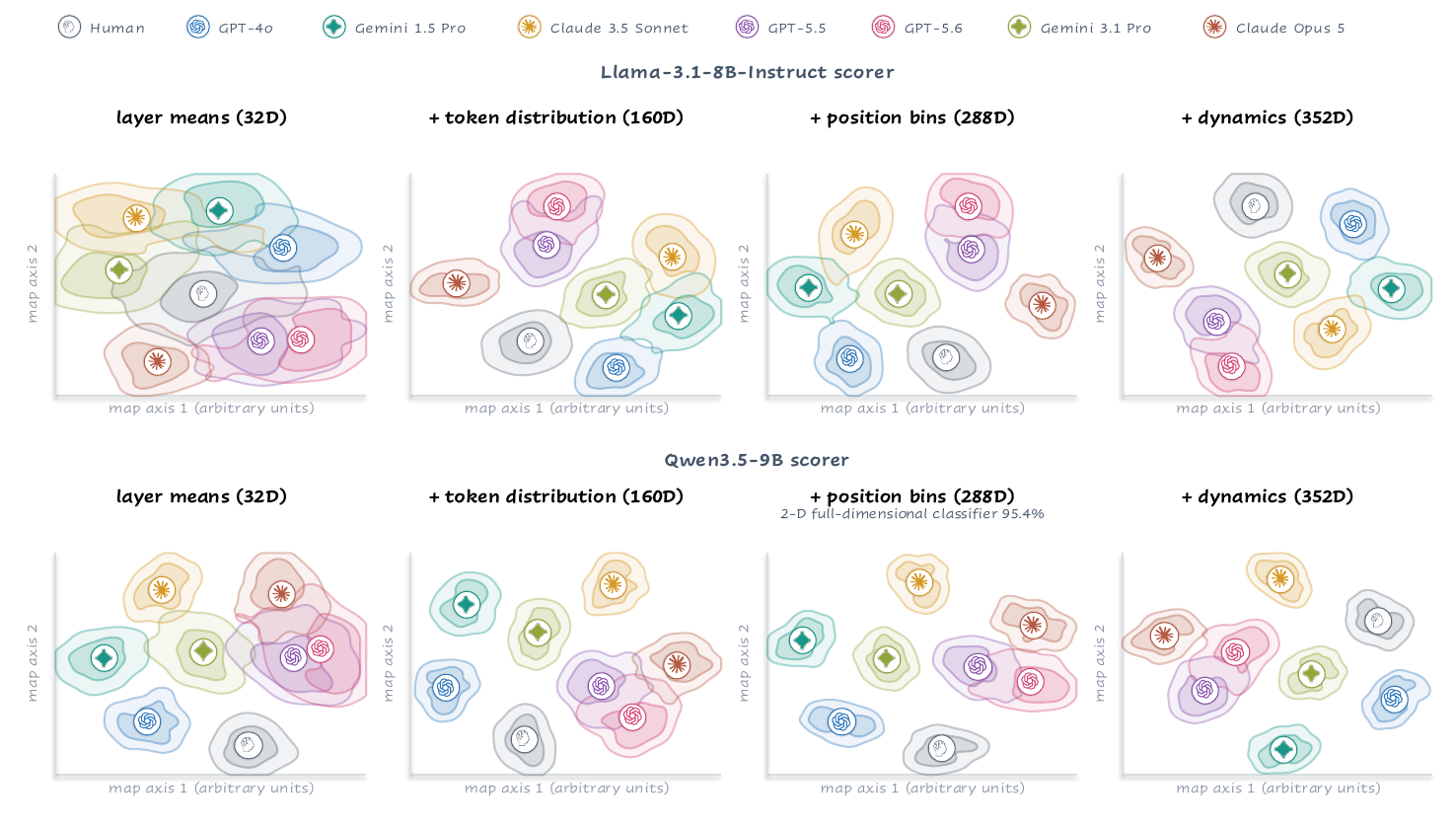}
\caption{Held-out reviews of the eight sources under the Llama scorer, for the four feature sets of Table~\ref{tab:attribution-features}. Each panel is a nonlinear (t-SNE) map of the reviews' scores on the seven LDA discriminant axes fitted on calibration papers; contours enclose 90\% and 50\% of each source's estimated density, dots mark centroids. Map distances are not calibrated: only cluster overlap is meaningful. Richer token statistics separate Human and the three legacy generators from the newer ones; GPT-5.5 and GPT-5.6 remain overlapped and account for most residual errors.}
\label{fig:generator_clusters_umap_all}
\end{figure}

\subsection{Response-Length Sensitivity}
\label{app:length-ablation}

\paragraph{Review detection.}
Although each BTE coordinate is already averaged over response tokens, finite-length effects may remain. On the Original test set, the within-class Spearman correlations between mean full-profile BTE and token count are $0.448/0.142$ for Llama and $0.122/0.040$ for Qwen (human/generated). We therefore test the \emph{full 32D} representation, which retains absolute BTE magnitude. For each layer $\ell$, we fit $e_{i\ell}=\alpha_\ell+\beta_\ell\log(1+n_i)+\epsilon_{i\ell}$ on the training reviews and replace each coordinate by its residual, $e'_{i\ell}=e_{i\ell}-\hat\alpha_\ell-\hat\beta_\ell\log(1+n_i)$. The same train-fitted coefficients transform the test reviews, and the residualizer is refit inside each training fold used to select the logistic-regression regularization. Table~\ref{tab:review-length-ablation} shows that length alone is weak (macro-AUROC $0.512$-$0.538$), while classifiers fitted to the residual profiles differ from their raw-profile AUROC by at most $0.008$ across all settings and both evaluation arms.

\begin{table}[tbp]
\centering
\small
\setlength{\tabcolsep}{4.2pt}
\begin{tabular}{@{}llcccccc@{}}
\toprule
& & \multicolumn{3}{c}{\textbf{Original}} & \multicolumn{3}{c}{\textbf{Rewrite}} \\
\cmidrule(lr){3-5}\cmidrule(l){6-8}
\textbf{Setting} & \textbf{Scorer} & \textbf{Length} & \textbf{Raw 32D} & \textbf{Residual 32D} & \textbf{Length} & \textbf{Raw 32D} & \textbf{Residual 32D} \\
\midrule
A & Llama & 0.5383 & 0.7416 & 0.7489 & 0.5349 & 0.7197 & 0.7233 \\
A & Qwen  & 0.5344 & 0.8911 & 0.8848 & 0.5308 & 0.8640 & 0.8574 \\
B & Llama & 0.5123 & 0.9094 & 0.9167 & 0.5119 & 0.9046 & 0.9127 \\
B & Qwen  & 0.5344 & 0.9940 & 0.9933 & 0.5308 & 0.9934 & 0.9925 \\
C & Llama & 0.5383 & 0.9680 & 0.9695 & 0.5349 & 0.9654 & 0.9680 \\
C & Qwen  & 0.5344 & 0.9980 & 0.9976 & 0.5308 & 0.9979 & 0.9976 \\
\bottomrule
\end{tabular}
\caption{Response-length ablation for review detection. Length is a class-balanced logistic detector using only $\log(1+n)$; Residual 32D removes the train-fitted linear length component from every BTE coordinate. Values are macro-AUROC over the same conference-year $\times$ generator as Table~\ref{tab:bte-review-three-settings}.}
\label{tab:review-length-ablation}
\end{table}

\paragraph{Difficulty.}
We perform the corresponding analysis on the 500 MATH-500 solution responses. Difficulty is correlated with response length ($\rho=0.512$), and BTE itself is length-sensitive, making response length a potential contributor to the observed difficulty association. We therefore compute partial Spearman correlations controlling for response length, with problem length included as an additional sensitivity analysis.
As shown in Table~\ref{tab:math-length-ablation}, controlling for response length substantially attenuates the association between BTE and difficulty, indicating that solution length explains part of the original signal. Nevertheless, a residual association remains for several scorers, most clearly for Llama and OpenMath2, even after controlling for both response and problem length. Thus, BTE should not be interpreted as a length-independent measure of difficulty; rather, it captures a mixture of difficulty-related factors, including but not limited to the amount of solution computation reflected by response length.

\begin{table}[htbp]
\centering
\small
\begin{tabular}{@{}lcccc@{}}
\toprule
\textbf{Scorer} & $\boldsymbol{\rho(E,n)}$ & $\boldsymbol{\rho(E,d)}$ & $\boldsymbol{\rho(E,d\mid n)}$ & $\boldsymbol{\rho(E,d\mid n,m)}$ \\
\midrule
Llama-3.1-8B-Instruct       & 0.576 & 0.444 & 0.212 &  0.151 \\
OpenMath2-Llama3.1-8B      & 0.431 & 0.391 & 0.219 &  0.191 \\
DeepSeek-R1-Distill-Llama-8B & 0.106 & 0.143 & 0.104 &  0.082 \\
Nemotron-Nano-8B-v1        & 0.219 & 0.164 & 0.061 & $-0.016$ \\
\bottomrule
\end{tabular}%
\caption{Length sensitivity on MATH-500. $E$ is mean Total BTE over layers 12-19, $d$ is difficulty level, $n$ is response-token count, and $m$ is problem-token count. The response-length-adjusted association is significant for the first three scorers ($p\leq0.021$), but not Nemotron ($p=0.173$); after controlling both lengths, only Llama and OpenMath2 remain significant ($p<0.001$).}
\label{tab:math-length-ablation}
\end{table}

\begin{table}[htbp]
\centering\small
\setlength{\tabcolsep}{4pt}
\begin{tabular}{llrrrrrc}
\toprule
\textbf{Checkpoint} & \textbf{Step} & \textbf{Tokens} & \textbf{L0-3} & \textbf{L12-19} & \textbf{L28-31} & $\boldsymbol{\rho}$ & \textbf{95\% CI} \\
\midrule
Gaussian init (same architecture) & - & 0 & 0.130 & 0.013 & 0.007 & $-0.01$ & $[-0.10,\,+0.07]$ \\
\midrule
stage1-step150 & 150 & 1B & 0.123 & 0.006 & 0.002 & $-0.14$ & $[-0.23,\,-0.05]$ \\
stage1-step1000 & 1,000 & 5B & 0.139 & 0.012 & 0.007 & $-0.15$ & $[-0.24,\,-0.07]$ \\
stage1-step5000 & 5,000 & 21B & 0.230 & 0.031 & 0.015 & $-0.12$ & $[-0.20,\,-0.03]$ \\
stage1-step12000 & 12,000 & 51B & 0.289 & 0.044 & 0.020 & $-0.14$ & $[-0.23,\,-0.06]$ \\
stage1-step24000 & 24,000 & 101B & 0.344 & 0.058 & 0.021 & $-0.15$ & $[-0.24,\,-0.06]$ \\
stage1-step47000 & 47,000 & 198B & 0.264 & 0.071 & 0.025 & $-0.14$ & $[-0.23,\,-0.05]$ \\
\midrule
stage1-step122000 & 122,000 & 512B & 0.129 & 0.076 & 0.035 & $-0.14$ & $[-0.23,\,-0.06]$ \\
stage1-step238000 & 238,000 & 999B & 0.105 & 0.070 & 0.041 & $-0.13$ & $[-0.22,\,-0.04]$ \\
stage1-step366000 & 366,000 & 1,536B & 0.101 & 0.067 & 0.047 & $-0.08$ & $[-0.17,\,+0.01]$ \\
stage1-step488000 & 488,000 & 2,047B & 0.102 & 0.067 & 0.052 & $-0.11$ & $[-0.19,\,-0.03]$ \\
stage1-step610000 & 610,000 & 2,559B & 0.106 & 0.071 & 0.056 & $-0.08$ & $[-0.16,\,+0.00]$ \\
stage1-step732000 & 732,000 & 3,071B & 0.111 & 0.073 & 0.058 & $-0.07$ & $[-0.16,\,+0.02]$ \\
stage1-step855000 & 855,000 & 3,587B & 0.119 & 0.079 & 0.061 & $-0.04$ & $[-0.14,\,+0.04]$ \\
stage1-step928646 & 928,646 & 3,896B & 0.126 & 0.083 & 0.064 & $-0.01$ & $[-0.11,\,+0.07]$ \\
main (annealed) & - & 3,946B & 0.126 & 0.093 & 0.059 & $+0.03$ & $[-0.06,\,+0.12]$ \\
\midrule
Llama-3.1-8B (base) & - & $\sim$15T & 0.241 & 0.167 & 0.105 & $+0.36$ & $[+0.28,\,+0.43]$ \\
Llama-3.1-8B-Instruct & - & $\sim$15T & 0.230 & 0.168 & 0.097 & $+0.44$ & $[+0.37,\,+0.51]$ \\
\bottomrule
\end{tabular}
\caption{OLMo-2-7B checkpoints (HuggingFace branches of \texttt{allenai/OLMo-2-1124-7B}) with mean Total-BTE on MATH-500 solutions in three layer bands and the Spearman correlation $\rho$ between middle-layer BTE and MATH level (paper-level bootstrap CI). The middle block is the set shown in Fig.~\ref{fig:sft-shape}C; Llama-3.1-8B is listed for reference.}
\label{tab:olmo-checkpoints}
\end{table}

\FloatBarrier
\subsection{Fine-tuning details}
\label{app:sft-details}

\paragraph{Data.}
Solutions come from PRM800K~\citep{lightman2023lets}: human-verified solution trajectories to MATH problems, levels 1-5, at most 2,048 tokens including the prompt, diagram questions removed. Every problem used in our evaluations (MATH-500, the development set and the GSM8K probe) is excluded. Several verified trajectories of the same problem may appear. The full pool has 9,704 solutions (906 / 2,067 / 2,410 / 2,314 / 2,007 for levels 1-5). The two restricted arms are drawn from the same pool: 4,320 of its 4,321 level-4-5 solutions (2,314 / 2,006 for levels 4 / 5) and 2,968 of its 2,973 level-1-2 solutions (903 / 2,065 for levels 1 / 2); the arm builder's phase-2-first fill dropped one and five phase-1 trajectories as surplus.

\paragraph{Training.}
We fine-tune all parameters of Llama-3.1-8B-Instruct with the loss on solution tokens only, for one epoch in random order: constant learning rate $10^{-5}$ without warm-up, AdamW with weight decay 0.01, effective batch size 8 (two GPUs, gradient accumulation 4), context 2,048 tokens, fixed seed. The dose experiment saves checkpoints after 1,000, 5,000 and 9,704 solutions; the levels 4-5 arm is also saved after 2,968 solutions to match the size of the levels 1-2 arm. Runs take 57, 25 and 17 minutes on two NVIDIA GPUs with 90\,GB of memory each. The random-initialisation control uses the OLMo-2-7B architecture with Gaussian initialisation (standard deviation 0.02), fine-tuned with the same recipe at learning rate $10^{-4}$ on 2,000 solutions.

\paragraph{Pretraining checkpoints.}
Fig.~\ref{fig:sft-shape}C uses the released stage-1 checkpoints of OLMo-2-7B~\citep{OLMo2}, the final annealed model, and a Gaussian-initialised model of the same architecture; Table~\ref{tab:olmo-checkpoints} lists them together with the earlier checkpoints (1B-198B tokens) that locate the rise of the middle-layer BTE. These base models have no chat template, so the MATH-500 prompt is given as plain text followed by ``Solution:''; all other scoring settings are unchanged.

\end{document}